\documentclass{article}
\usepackage[preprint]{neurips_2021}

\usepackage{microtype}
\usepackage{graphicx}
\usepackage{booktabs} 

\usepackage[utf8]{inputenc} 
\usepackage[T1]{fontenc}    
\usepackage[breaklinks=true]{hyperref}       
\usepackage{url}            
\usepackage{booktabs}       
\usepackage{amsfonts}       
\usepackage{nicefrac}       
\usepackage{microtype}      
\usepackage{graphicx}
\usepackage{xcolor}
\usepackage{subfig}
\usepackage{booktabs} 
\usepackage{amsmath,amssymb, bm, bbm, amsthm}
\usepackage{wrapfig,lipsum,booktabs}

\DeclareMathOperator*{\argmin}{arg\,min}
\DeclareMathOperator{\E}{\mathbb{E}}

\usepackage{algpseudocode}
\usepackage{algorithm}
\newtheorem{definition}{Definition}
\newtheorem{proposition}{Proposition}
\newtheorem{corollary}{Corollary}
\newtheorem{lemma}{Lemma}
\theoremstyle{remark}
\newtheorem{remark}{Remark}

\title{Uncertainty Characteristics Curves: A Systematic Assessment of Prediction Intervals}

\author{
Ji\v{r}\'\i\, Navr\'atil\\
\texttt{jiri@us.ibm.com}\\
IBM Research
\And
Benjamin Elder\\
\texttt{benjamin.elder@ibm.com}\\
IBM Research
\And
Matthew Arnold\\
\texttt{marnold@us.ibm.com}\\
IBM Research
\And
Soumya Ghosh\\
\texttt{ghoshso@us.ibm.com}\\
IBM Research
\And
Prasanna Sattigeri\\
\texttt{psattig@us.ibm.com}\\
IBM Research
}

\begin{document}

\maketitle

\begin{abstract}
Accurate quantification of model uncertainty has long been recognized as a fundamental requirement for trusted AI. In regression tasks, 
uncertainty is typically quantified using prediction intervals calibrated to a specific operating point, making evaluation and comparison across different studies difficult.
Our work leverages: (1) the concept of operating characteristics curves and (2) the notion of a gain over a simple reference, to derive a novel operating point agnostic assessment methodology for prediction intervals.
The paper describes the corresponding algorithm, provides a theoretical analysis, and demonstrates its utility in multiple scenarios. We argue that the proposed method addresses the current need for comprehensive assessment of prediction intervals and thus represents a valuable addition to the uncertainty quantification toolbox.   
\end{abstract}

\section{Introduction}\label{Sec:Intro}

The ability to quantify the uncertainty of a model is one of the fundamental requirements in trusted, safe, and actionable AI \cite{Arnold2019_factsheets, Jiang2018_trust, Begoli2019}. 
Numerous methods of generating uncertainty bounds (referred to as prediction intervals, or error bounds) 
have been proposed in statistics and machine learning literature. 

Evaluating the quality of uncertainty bounds, however, remains challenging. While metrics such as the measure of coverage and the likelihood are popular, they conflate the quality of the prediction intervals (PI) with the difficulty of the predictive task at hand\footnote{We give an example of such conflation in the Appendix}. These metrics may be too optimistic (pessimistic) for trivial (challenging) datasets with an easily predictable (hard to quantify) uncertainty and can lead to misleading conclusions. Moreover, tools to compare multiple models producing bounds across different \emph{operating points} (OP) are scarce. We contend that to measure the quality of uncertainty bounds reliably, we need tools that are both OP and dataset independent. 

{\bf OP independence.}
The importance of OP-independent (OP-comprehensive) evaluation metrics is well understood, as demonstrated by techniques such as ROC curves \cite{Fawcett06} becoming standard practice in numerous areas of AI involving 
classification, detection, and other tasks. 
In the context of PI, we define the term OP loosely as a specific setting producing a certain value of either mean coverage or bandwidth (a formal definition will be given in Section \ref{Sec:UCC}). 
However, most evaluations of prediction intervals are still performed in a OP-dependent way, evaluating the quality at one (or a small number of) operating point(s).

{\bf Dataset independence}
When possible, metrics should capture the effectiveness of the technique being evaluated, rather than characterize the underlying dataset used in the evaluation.  Techniques that contextualize a measurement by comparing it to an intuitive reference 
offer themselves to address this challenge (for instance, the well-known coefficient of determination, or $R^2$, in statistics \cite{Steel1960} expresses 
an observed variance relative to an overall variance to be explained, thus yielding a measure comparable across
datasets as well as algorithms). 

This paper proposes a methodology for evaluating prediction intervals that addresses both of these issues.  First, we introduce the Uncertainty Characteristics Curve (UCC), which leverages the well known concepts of operating characteristic curves to enable OP-independent evaluation. Second, we introduce the notion of a {\em gain} over a simple reference, which more intuitively captures the quality of a prediction interval and makes it comparable across different models.  These two approaches can be combined to produce a novel metric that we believe is a valuable addition to the prediction interval assessment toolbox. 

The contributions of this paper are summarized as follows: 
\begin{itemize}
    \item Developing the Uncertainty Characteristics Curve (UCC) as a tool to comprehensively assess the quality of prediction intervals, whether Bayesian or frequentist in nature.
    \item Providing a probabilistic interpretation of the area under the UCC.
    \item Proposing a gain metric that allows for a cross-model comparison.
    \item Releasing the corresponding code to the research community.
\end{itemize}

\begin{figure}[htbp!]
       \centering
       \includegraphics[height=3.5cm, width=8cm]{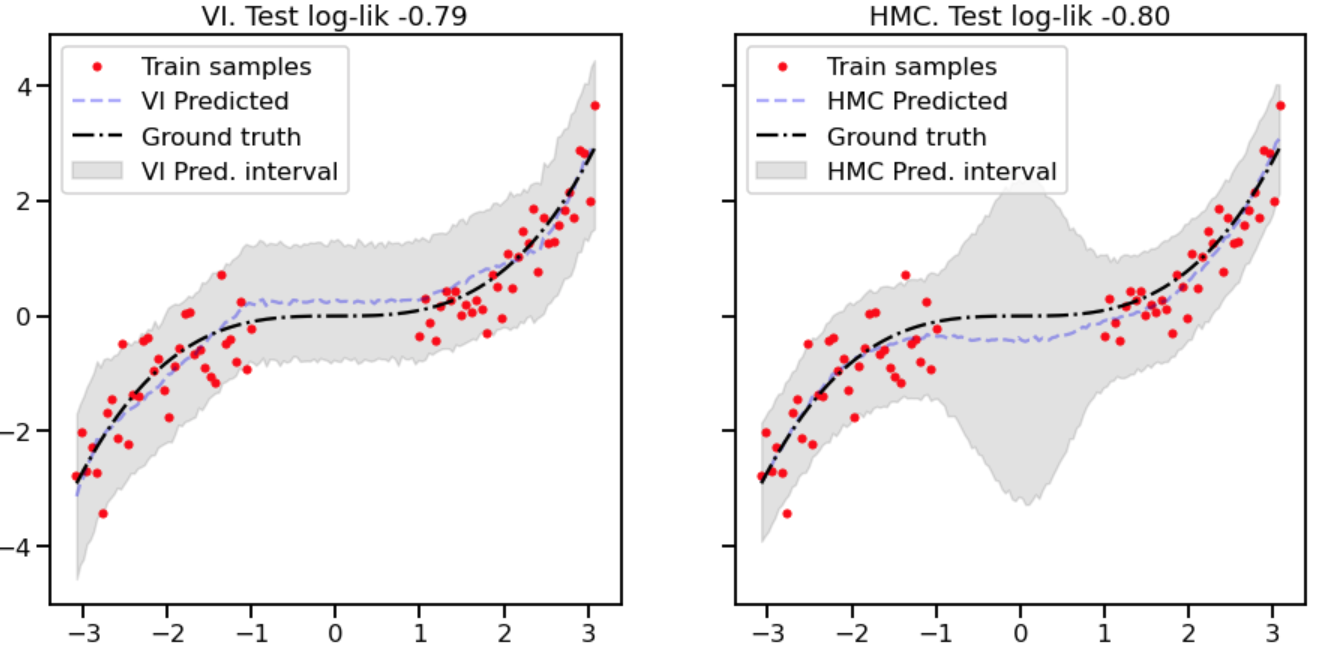}
       \caption{Bayesian Neural Networks applied to a synthetic regression task: Variational Inference (VI) on the left and Hamiltonian Monte Carlo (HMC) on the right. The HMC obtains a lower likelihood. }
       \label{Fig:IntroGapExample}
\end{figure}

We want to briefly introduce an illustrative example to underscore the complexities in prediction interval assessment. Figure \ref{Fig:IntroGapExample} shows an application of a Bayesian Neural Network (BNN) 
to a one-dimensional synthetic regression task involving  
a data gap in the region $[-1,1]$, similar to \cite{yao2019quality}. 
The example uses two inference variants to approximate the intractable BNN posterior: a computationally efficient black-box variational inference (VI) and a Hamiltoninan Monte Carlo (HMC) method - largely considered the gold standard for BNN inference\cite{neal2011mcmc,Foong2020}. A significant widening in the uncertainty occurs with the HMC in the data-gap region, which appeals to our intuition. However, when assessing the quality using the log-likelihood on test data, the VI scores higher than the HMC - 
an example of how likelihood may be an inconclusive or a misleading indicator. 
We will return to this example in more detail later on and show that our proposed metric 
points clearly in favor of the HMC variant.

\section{Method}

\subsection{Basic Metrics}
\label{Sec:Metrics}
Suppose there are two components of a model, one generating target predictions, the other assigning uncertainty bounds associated with each such target. Let $V=[Y, \hat{Y}, \hat{Y}^l, \hat{Y}^u]^T\in\mathbb{R}^4$ 
denote a multivariate random variable comprising the ground truth, $Y$, the target prediction, $\hat{Y}$,
the lower uncertainty bound, $\hat{Y}^l=\hat{Y}-\hat{Z}^l$, and the upper uncertainty 
bound, $\hat{Y}^u=\hat{Y}+\hat{Z}^u$.
The variables $\hat{Z}^l, \hat{Z}^u$ refer to the predicted uncertainty {\em bands}, i.e., the positive deviates from $\hat{Y}$. For simplicity we assume the task involves one-dimensional output, however, all subsequent explanations  
readily generalize to a multi-dimensional setting. 
Let $p_V$ denote the probability density of $V$ and $\mathbf{v}=\{v_1, ..., v_N\}$ a set of $N$ samples from $p_V$, where $v_i=[y_i, \hat{y}_i, \hat{y}_i^l, \hat{y}_i^u]^T$. 

We need to consider two cost aspects arising in the assessment of prediction intervals, roughly speaking:  (1) what is the extent of observations falling outside the uncertainty bounds (miss rate), and (2) the width of the bounds. An optimal bound captures all of the ground truth (or a calibrated fraction thereof) while being least excessive in terms of its average bandwidth.
We define the following metrics as expectations (and their corresponding empirical estimates):
\vspace{-.1cm}
\begin{equation}
    \mbox{Miss rate: }\qquad\rho(V)=\E_{p_V}\left[\mathbf{1}_{Y\notin[\hat{Y}^l, \hat{Y}^u]}\right]\,,
    \qquad \hat{\rho}(\mathbf{v}) =
    \frac{1}{N}\sum_{i:y_{i}\notin[\hat{y}_{i}^l, \hat{y}_{i}^u]} 1 \label{Eq:Missrate}
\end{equation}
\vspace{-.2cm}
\begin{equation}
    \mbox{Bandwidth: }\qquad\beta(V)= \frac{1}{2}\E_{p_V}\left[ \hat{Y}^u-\hat{Y}^l \right]\,,
    \qquad \hat{\beta}(\mathbf{v}) =
    \frac{1}{2N}\sum_{i=1}^{N} \hat{y}_{i}^u-\hat{y}_{i}^l \label{Eq:Bandwidth}
\end{equation}
Note that $\rho(V)$ is the complement of the well-known Prediction Interval Coverage Probability (PICP) and the bandwidth is half of the Mean Prediction Interval Width (MPIW).
In Appendix \ref{App:Sec:AddMetrics}, we define two additional metrics also suitable for practical applications. In favor of coordinate axes consistency, all metrics are defined as costs (hence our choice of Miss Rate 
over the PICP). 

In case the variables $Y, \hat{Y}, \hat{Z}^l$, and $\hat{Z}^u$ are multi-dimensional, the above metrics may be calculated as averages over the individual dimensions.

\subsection{Relative Gain}\label{Sec:RelativeGain}

While we want to assess the quality of the {\em uncertainty}, most 
standard metrics (e.g., likelihood, bandwidth, etc.) compound  
uncertainty bands with actual target predictions. As a consequence, these metrics, in their raw form, 
are not 
comparable across different predictors (an example of such confounding is the log-likelihood loss, see Appendix). To address this issue, we 
look for a relative gain of the PI metrics over some simple, intuitive reference. A good choice of such a reference 
are constant bands, i.e. $\forall i:\hat{z}_i^l=const, \hat{z}_i^u=const$. Given target predictions, 
$\hat{y_i}$, such a non-informative (null) baseline is always possible to generate. 

Figure \ref{Fig:GainExample} illustrates the issues in comparing raw metrics. 
\begin{figure*}
      \centering
      \includegraphics[height=4.5cm]{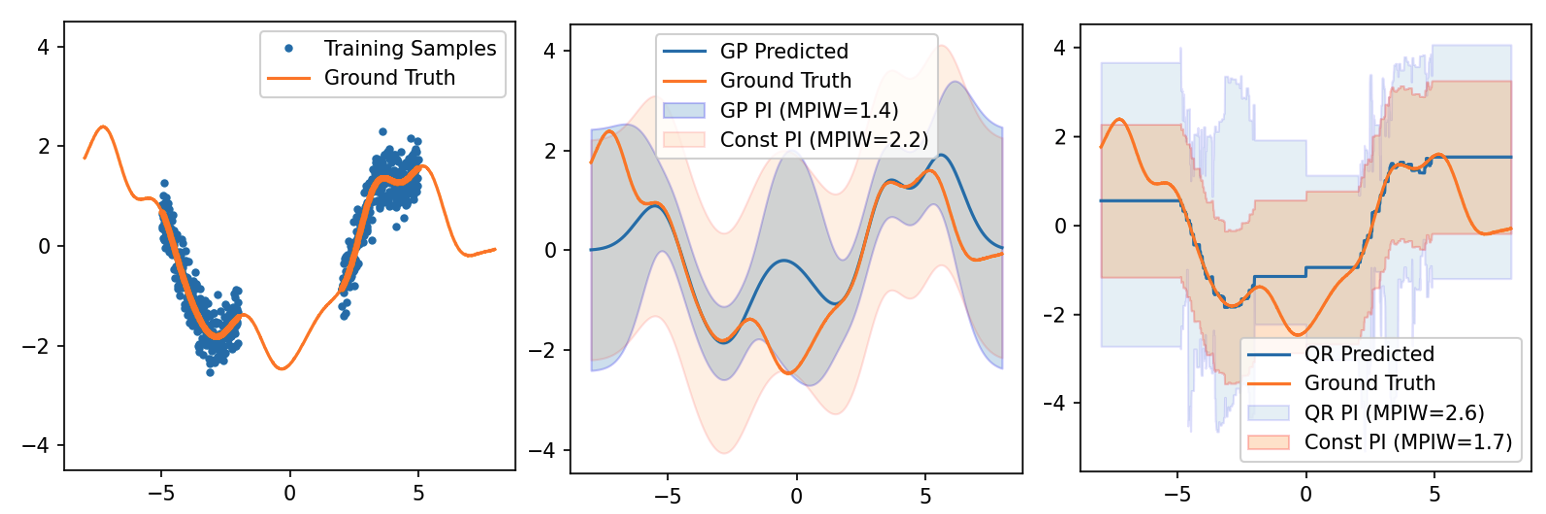}
      \caption{Example of a Gaussian Process (GP, center) and a Gradient Boosting Tree Quantile Regressor (QR, right) generating Prediction Intervals (PI) on a synthetic task (Left), along with constant-band
      references. All PIs are calibrated to a Miss Rate 5\% and the corresponding mean bandwidth values are also shown (MPIW).} 
      \label{Fig:GainExample}
\end{figure*}
In this example, the ground truth (GT) is generated by randomly drawing from a zero-mean Gaussian Process (GP) with an RBF kernel, $f(x)\sim GP(\mathbf{0},K(x,x^T)$. The training set is obtained as a noisy sample $y(x)\sim \mathcal{N}(f(x), 0.1)$ in the regions shown in the left plot of Figure \ref{Fig:GainExample}. 
Two models are fit to the data: a GP regressor (center) and a Gradient Boosted Tree (GBR) with Quantile Regression loss (QR, right). 
The GP and QR PIs are shown as blue-shaded areas while 
the light-shaded areas correspond to the reference constant bands spanned around the target predictions. 
All PIs are calibrated to incur a miss rate of 5\%. Several observations can be made: 
(1) The average bandwidth (MPIW) of the GP (1.4) is lower than that of the QR (2.6). 
(2) The constant-band MPIW, albeit at the same miss rate, differs across the cases (2.2 for the GP vs. 1.7 for the QR). 
(3) The GP PI are 36\% better than the reference and the QR PI are 53\% worse than their reference.  
The last observation exercises the relative notion (gain): the GP PI are clearly superior 
as they capture the GT far more efficiently than the constant band. In contrast, the QR predictor 
generates PI that seem to be inefficient, compared to the simple reference. 
While the individual MPIW values are not comparable (due to the different predictors involved), we can 
use the relative gain to rank the quality of the PI. A method achieving a high gain is obviously valuable. 
A method producing zero gain may not be necessarily faulty but the task at hand may just be homoskedastic 
in nature, suggesting to revise our model correspondingly. Finally, negative gains indicate an inferior model.  

One essential question remains open: 
How will the gain assessment change across different OPs (e.g. high, medium, low miss rate)? 
The answer lies at the foundation of the tool described next. 

{\subsection{The Uncertainty Characteristics Curve}
\label{Sec:UCC}}

\paragraph{Scaling}\label{Sec:Scaling}

In general, the goal of a calibration step is to transform the bounds such that a certain proportion of future observations, in expectation, fall within these bounds. 
Numerous calibration techniques have been developed in the context 
of regression, e.g., \cite{KuleshovAccurateUncertainties, SongDistributionCalibration2019}. Our considerations rely on an additive-error model from which scaling emerges as fundamental: We assume the ground truth $Y$ distributes as $Y\sim F(\hat{Y}, \hat{Z})$ where $F$ belongs to the {\em location-scale} family \cite{Rinne2011}, $\hat{Y}$ is an unbiased estimate of its mean, and $\hat{Z}>0$ its scale. Then
\begin{equation}
    Y=\hat{Y}+\hat{Z}\epsilon \label{Eq:linearrormodel}
\end{equation}
with $\epsilon\in\mathbb{R}$ a random error variable with $\epsilon \sim F(0, 1)$. $\hat{Z}$ represents the (symmetric) uncertainty\footnote{The argument for the case of asymmetric uncertainty is similar}.
A standard scale-calibration approach considers the variable $\frac{y-\hat{y}}{\hat{z}}\sim F(0,1)$ from which a desired quantile can be estimated. For instance, the quantiles $q_{0.025}$, $q_{0.975}$ can be used to obtain the prediction interval for a sample $\left[\hat{y}+\hat{z}\cdot q_{0.025}, \hat{y}+\hat{z}\cdot q_{0.975}\right]$. 
In this model, the quantile $q$ plays a {\em scaling} role, i.e., the predicted uncertainty bound $\hat{z}$ can be scaled up or down depending on a desired expected miss rate. We adopt the scaling operation as a fundamental step behind the Uncertainty Characteristics Curve defined in the following section. Here, the location-scale (LS) assumption
is a necessary limitation, however, a broad spectrum of distributions occurring in practice are genuinely LS and many others can be transformed to become LS \cite{Rinne2011}, therefore we consider this assumption mild.  

\paragraph{The Curve}
Given the scaling model above, we develop an assessment tool for prediction intervals. In this context, a single scaling value and the corresponding costs of miss rate and bandwidth, characterize a particular Operating Point (OP). A set of OPs can be obtained by varying the scale applied to ${\hat{Z}}^l$ and ${\hat{Z}}^u$ over the entire range relevant to the data $\mathbf{v}$. 
With $k> 0$ denoting the scaling variable, we recast the dataset $\mathbf{v}$ as functions of $k$: 
\begin{equation}
    \mathbf{v}(k) = \left\{v_i(k)\right\}_{1\leq i\leq N}=\nonumber\\
    =\left\{[y_i, \hat{y}_i, \hat{y}_i-k\hat{z}_i^l, \hat{y}_i+k\hat{z}_i^u]^T\right\}_{1\leq i\leq N}.
\end{equation}
where $\hat{z}_i^l=\hat{y}_i - \hat{y}_i^l$ and $\hat{z}_i^u=\hat{y}_i^u - \hat{y}_i$ are the predicted bands for a sample $i$. 
To further simplify the notation, we use a shorthand to write the bandwidth and the miss rate as functions of $k$
\begin{equation}
    \hat{\beta}(k):=\hat{\beta}(\mathbf{v}(k))\enspace\mbox{and}\enspace
    \hat{\rho}(k):=\hat{\rho}(\mathbf{v}(k)).
    \label{Eq:MetricShorthand}
\end{equation} 
E.g., $\hat{\rho}(k)$ gives the average miss rate of the PI after re-scaling all uncertainty bands, 
$\hat{z}_i^{l,u}$ using $k$. It can be 
readily observed that $\hat{\beta}(k)=c\cdot k$ with $c$ a constant depending on the original dataset. 

We now define the Uncertainty Characteristics Curve. 
\begin{definition}
    The Uncertainty Characteristics Curve (UCC) is a set of operating points $\left\{\left(\hat{\beta}(k), \hat{\rho}(k)\right)\right\}_{k\in S}$
    forming a bidimensional graph with the x-axis corresponding to the bandwidth and the y-axis  
    to the miss rate, and with $S$ a set of desirable scales. 
\end{definition}

The UCC graph shows the {\em trade-off} characteristics between the two cost metrics as a function of the calibration $k$.
As with the ROC \cite{Fawcett06}, a UCC can be parametric, however, in most practical cases is considered non-parametric with the cardinality of $S$ determined by the number of observations.

An extension of the UCC to include additional axes metrics is described in the Appendix. 

Given a dataset of size $N$, Algorithm \ref{Alg:UCC} (see Appendix) shows an efficient computation of the UCC with a complexity of $O(N^2)$. 

\begin{figure*}[hbt]
\centering
\begin{minipage}{.4\textwidth}
  \centering
  \includegraphics[height=3.5cm, width=5.5cm]{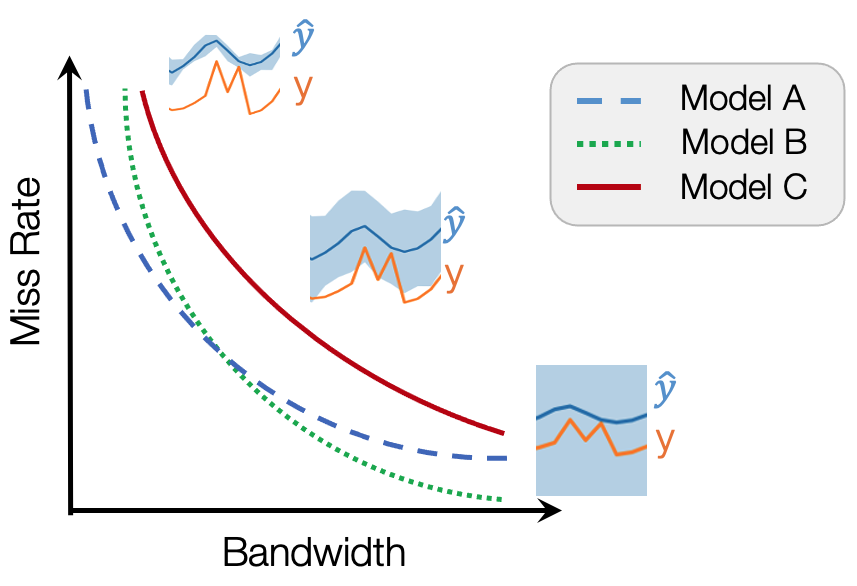}
  \captionof{figure}{An illustrative example of multiple UCCs obtained for different models}
  \label{Fig:UCCExample1}
\end{minipage}\hspace{1.5cm}%
\begin{minipage}{.4\textwidth}
  \centering
  \includegraphics[height=3.4cm, width=5cm]{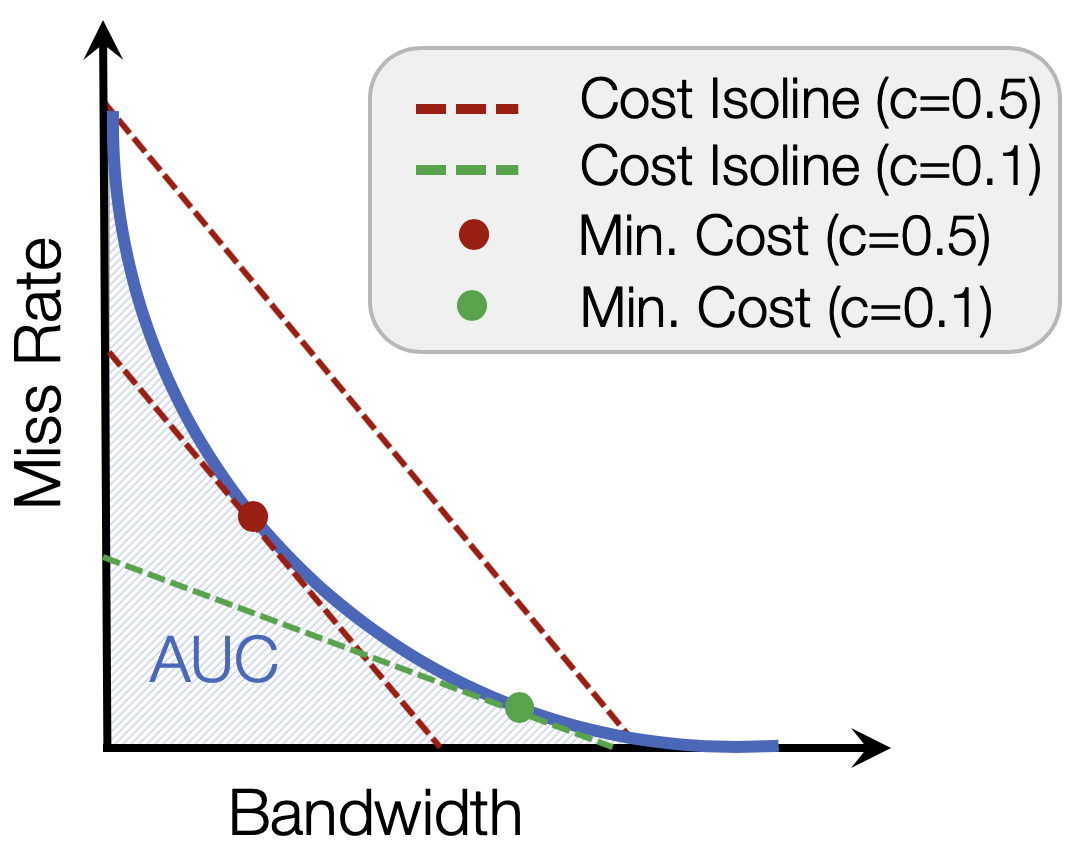}
  \captionof{figure}{Illustration of the Area Under Curve (AUC) and Cost within the UCC graph }
  \label{Fig:UCCExample2}
\end{minipage}
\end{figure*}
An illustrative example of a UCC graph is given in Figure \ref{Fig:UCCExample1} showing three curves corresponding to three different models generating uncertainty bounds $\hat{z}$ around the same target predictions $\hat{y}$. Illustrative icons characterizing low, middle, and high miss rate regimes are also shown. Each UCC curve reflects the operating characteristics of its model by showing a trade-off between the two costs. In this example, the curve for model C dominates A and B and is therefore inferior as it implies higher bandwidth is needed to achieve any given miss rate. In contrast, the model A appears superior to B in a low bandwidth range, while B outperforms A in a low miss rate area. Note that each curve eventually intercepts both axes reaching a zero value. 

Multiple curves on a common graph are comparable as long as the ground truth, $y$, and target 
predictions, $\hat{y}$, are shared across these multiple models - as is the case, for example, 
with a model's PI along with with a constant-band reference. The dependence on $\hat{y}$ is 
due to the fact that the bounds are obtained by adding scaled bands, $kz$, to $\hat{y}$ resulting in differently shaped bounds for different $\hat{y}$ thus inducing different intersection with the observations. We give an illustrative example of this dependence in the Appendix.

\paragraph{Cost Function}
Considering the cost trade-off between the two axes at a particular operating point, it is useful to define a function combining the two in a meaningful way. The simplest example is a linear cost function
\begin{equation}\label{Eq:CostFunction}
    C(k) = c \hat{\beta}(k) + (1-c) \hat{\rho}(k), \enspace c\in[0,1]
\end{equation}
that uses an application-dependent factor, $c$, to focus on a specific area of the operating range (e.g., low miss rate area). On the UCC coordinate system, $C(k)=const$ shows as an isocost line (see Figure \ref{Fig:UCCExample2}) whose slope is proportional to $-c/(1-c)$. A minimum achievable cost, $C(k^*)$ with $k^*=\argmin_k C(k)$, is an intersection of a model's UCC and the minimal isocost as illustrated in the example. In this context, the UCC provides for a visual assessment between the original OP cost and the optimum cost as well as gives the scaling $k^*$ needed to reach that optimum. 

In the Appendix, we point out interesting relationship between the above cost function and the  
well-known measures of Mean Absolute Error and Interval Score.

\paragraph{Area Under the UCC (AUUCC)}

It is desirable to define a summary metric capturing the overall quality of a PI model. Given that the UCC coordinates correspond to costs, a sensible choice is the area under the curve (or AUUCC), as shown in Figure \ref{Fig:UCCExample2}. Models predicting bounds that incur lower cost across the entire 
operating range will produce a lower AUUCC. Thus, in absence of a pre-determined operating point, the AUUCC measure can be a useful OP-agnostic indicator of the predictor's quality. Alternatively, if a certain range of the cost, say, the miss rate values is anticipated, a {\em partial} AUUCC focusing on that range can be determined, similar to the notion of partial ROC AUC \cite{PartialAUC2013}.

Unlike with the ROC AUC analysis, the range of the AUUCC depends on the PI range and is therefore data dependent. This underscores the need for normalization as discussed in Section \ref{Sec:RelativeGain}.

\paragraph{AUUCC Gain}

We now return to the issues discussed in 
Section \ref{Sec:RelativeGain}.  
By using the AUUCC in the relative gain calculation 
the question about the specific OP is integrated out, via the area calculation. 
Let $A_M$ represent the AUUCC of a model and let $A_{Const}$ refer to the constant-band counterpart.
The AUUCC Gain is then defined as:
\begin{equation}
    G_M = \frac{A_{Const}-A_M}{A_{Const}}\cdot 100\%\label{Eq:AUUCCGain}.
\end{equation}
As discussed with examples from Figure \ref{Fig:GainExample}, 
negative gains
are an indication of a model issue (misspecification, over-training, etc.). A positive 
value summarizes the overall OP-agnostic quality. 
The partial-AUUCC gain is calculated similarly via Eq. (\ref{Eq:AUUCCGain}) using partial AUUCC values. 

\paragraph{Interpreting the AUUCC}

A probabilistic interpretation of the area under the ROC is well known \cite{Fawcett06}. 
Bi et al. \cite{Bi_REC_2003} also established a connection between the area under the Regression Error Characteristics (REC) curve and the expected error. 
In a similar vein, we derive a probabilistic interpretation for the AUUCC.
\begin{definition}(Critical Scale).\label{Def:CritScale}
    Given an observation $v_i=[y_i, \hat{y}_i, \hat{y}_i^l, \hat{y}_i^u]^T$, 
    a critical scale is a factor $k_i$ calculated according to 
    \vspace{-.5cm}
    \begin{equation}
        k_i := 
     \begin{cases}
      \frac{z_i}{\hat{z}_i^u}& \text{for}\,\,z_i\geq 0\\
      -\frac{z_i}{\hat{z}_i^l}& \text{otherwise}\\
     \end{cases}\label{Eq:ScaleDef}
    \vspace{-.2cm}
    \end{equation}
    where $z_i=y_i - \hat{y}_i$, $\hat{z}_i^l=\hat{y}-\hat{y}_i^l$, and $\hat{z}_i^u=\hat{y}_i^u-\hat{y}$.
\end{definition}
The critical value $k_i$ scales the active (lower or upper) error band, $\hat{z}_i^{l,u}$, so that it captures the observation $y_i$ with no excess. Note that this notion is also utilized in 
the Algorithm \ref{Alg:UCC}.

\begin{proposition}\label{Prop:missrate}
    Let $B$ denote a bandwidth random variable generated by the following procedure: (i) Randomly select 
    an observation, $v=[y, \hat{y}, \hat{y}^l, \hat{y}^u]^T$ according to $p_V$, (ii) determine its critical scale, $k$, via Eq. (\ref{Eq:ScaleDef}), and (iii) obtain the average bandwidth value via
    Eq. (\ref{Eq:MetricShorthand}) using the exact metric $b={\beta}(k)$. Let $p_B$ denote the probability density of $B$.
    The area under the UCC, calculated from a finite sample $\{v_1,..., v_N\}$, is an estimator of the expected value $\left<B\right>_{p_B}$. 
\end{proposition}

The Proposition \ref{Prop:missrate} states that the AUUCC estimates the expected bandwidth over the set of data-induced operating points. Consequently, for a given sample of predictions, $\{v_1,..., v_N\}$, the sample average of the corresponding bandwidth values, $\{b_1,..., b_N\}$, determined in Algorithm \ref{Alg:UCC} approximates the AUUCC. This connection is analogous to one pointed out by Bi et al \cite{Bi_REC_2003} for the REC curve. 

The proof of Proposition \ref{Prop:missrate} along with additional results as well as more discussion can be found in the Appendix. 

\paragraph{Significance Testing}
 Standard tools of significance testing are applicable to the UCC in a straight-forward manner. 
 If a pairwise comparison between two models in terms of the AUUCC is desired, the non-parametric 
 paired permutation test \cite{dwass1957_permutationtest} is applicable and was also used in our
 experiments.

\section{Related Work}

Uncertainty quantification in machine learning (ML) is a long-standing field of active research. 
Sources
of uncertainty are generally categorized as epistemic or aleatoric \cite{Kiureghian09, Kendall2017_whatuncertainty}. In classification tasks, uncertainty is expressed as a measure of confidence  
accompanying a result \cite{gal2015theoretically, Guo2017, Lakshminarayanan2017, ovadia2019trustinuncertainty}.
Combined with an optional calibration step, e.g. \cite{Zadrozny2002}, a quality assessment of such estimates relies on summary metrics, such as the Brier score \cite{Brier1951, Bradley2019, Lakshminarayanan2017, Guo2017, Zadrozny01obtainingcalibrated}, Expected Calibration Error \cite{Naeini15, Guo2017}, ROC-like metrics and  Accuracy-vs-Confidence curves \cite{Chen2019_linearprobes, Lakshminarayanan2017}.
Uncertainty in {\em regression} tasks involves estimating PIs 
(e.g., \cite{KoenkerQuantileRegression78, Papoulis1989}) as well as 
in state-of-the-art ML \cite{Nix1994_variancemodel, gal2015theoretically, Kendall2017_whatuncertainty}. 
However, the methodology of comparing their quality is relatively scarce, ranging  
from reliance on calibration and sharpness \cite{KuleshovAccurateUncertainties, Gneiting2007}, coverage  \cite{Oh2020_crowdcounting}, to using summary likelihood measures \cite{Lakshminarayanan2017}.

The Uncertainty Characteristic Curve (UCC) broadens the evaluation aspect drawing an analogy to the well-known ROC \cite{Fawcett06}.
The trade-off between two costs has been studied and applied previously \cite{Gneiting2007, Dunsmore1968, Shen2018_mentionsbandwidth, Tagasovska2019}. However, most reports rely on a specific OP during the assessment stage. 
A connection between the cost function (Eq. (\ref{Eq:CostFunction})) and the {\em Interval Score} \cite{Gneiting2007, Dunsmore1968} exists and is elaborated on in the Appendix.
In the context of regression, Bi et al. \cite{Bi_REC_2003} developed an assessment tool termed Regression Error Characteristic (REC) curve utilizing a constant tolerance band around a regression target. The REC allows for a comprehensive assessment of regressors. The UCC is conceived in a similar vein. Besides the different application and metrics used, 
the UCC fundamentally differs from the REC by not relying on a constant tolerance band but generalizing 
to an arbitrary tolerance band.
Finally, our work should be contrasted to calibration curves (also known as 
reliability diagrams) used for assessment primarily in prognostic aspects of classification tasks \cite{ReliabilityDiagrams2005}, and, more recently, in regression tasks \cite{tran2020methods}.
A calibration curve captures the amount of over- and under-confidence in the PI with respect to observed 
quantiles. While these curves vary the calibration setting 
there are two essential differences: (1) both axes are quantiles (i.e., there is no cost trade-off relationship), 
and (2) the actual quality ("accuracy") of the PI is not captured. Poor
PI can obtain a perfect calibration curve and vice versa. 

\section{Experiments}

In this section 
we present selected case studies that highlight relevant UCC use scenarios, also referring the reader to the
Appendix 
for an expanded report and configuration details. Result-reproducing notebooks
are provided in the Supplementary Material.

\vspace{-.3cm}
\paragraph{UCC Implementation}

The UCC was implemented in Python 3 and is available as a self-standing class providing methods 
to generate results shown in this paper and beyond. The UCC code along with an introductory exercise notebook is also part of the Supplementary Material, and will be made publicly available.




\subsection{Synthetic Data}
\label{Sec:SyntheticData}

We begin with a basic synthetic modeling example\footnote{https://scikit-learn.org/stable/auto\_examples/ensemble/{\allowbreak}plot\_gradient\_boosting\_quantile.html}, in which   
the function $x\sin x$ is mixed with an additive gaussian noise and is sampled randomly to create 5000 training points. Additionally, 1000 testing points are obtained from the non-noisy version via  
equidistant sampling. 
The training data are used to obtain a GBR target predictor using \cite{scikit-learn}.
The  prediction intervals were obtained using the quantile GBR method 
whereby a (1) well-tuned and a (2) under-parameterized ("weak") model were created.   
Besides constant-band prediction intervals, we also add a {\em random} symmetric error band drawn from a uniform distribution $\hat{z}_i\propto U(\frac{1}{3}\sigma_{\hat{y}}, 3\sigma_{\hat{y}})$ with $\sigma_{\hat{y}}$ the standard deviation of $\hat{y}$, representing a "worst-case" PI. 
Finally, representing an {\em ideal} prediction interval, an ``Epsilon-Perfect'' symmetric bound
was constructed as $\hat{z}_i=|\hat{y}_i-y_i|+\epsilon$, with $\epsilon$ being a small amount of additive noise (to visualize a curve rather than a single point). Thus, with proper scaling this bound will capture all of the observation 
at once, with no excess (modulo the $\epsilon$ noise), as it uses the ground truth. 
Figure \ref{Fig:SyntDatasetUCC} shows the resulting UCCs. Hollow circles mark the operating points calibrated 
to minimize the cost, Eq. (\ref{Eq:CostFunction}), with $c=0.1$. 
The ``$\epsilon$-perfect'' curve--the best achievable curve--shows up as a vertical line reaching the 
x-axis at a fixed positive bandwidth (around 0.85). This is a consequence of the bandwidth metric accounting for the minimum width necessary to capture the ground truth.
The remaining models rank as expected: the tuned GBR performs best, followed by the constant bound, the weak GBR, and the random bounds. A table with the full list of summary metrics can be found in the Appendix. 
All AUUCC values shown are significantly different at $p<0.01$, based on the pairwise permutation test \cite{dwass1957_permutationtest}.

\begin{figure*}
\centering
\begin{minipage}{.3\textwidth}
  \centering
  \includegraphics[width=4.7cm, height=3.5cm]{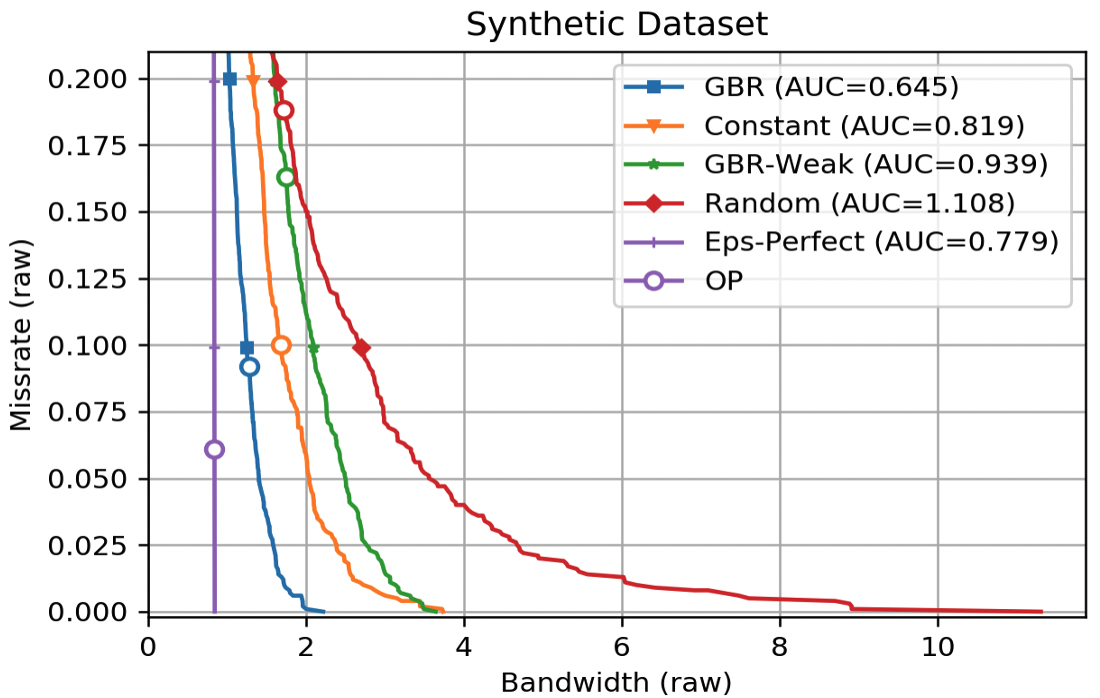}
  \captionof{figure}{UCC plot with curves for the 
      various prediction intervals on the synthetic dataset. OP stands for operating point.}
  \label{Fig:SyntDatasetUCC}
\end{minipage}\hspace{.5cm}%
\begin{minipage}{.27\textwidth}
  \centering
  \includegraphics[width=4.5cm]{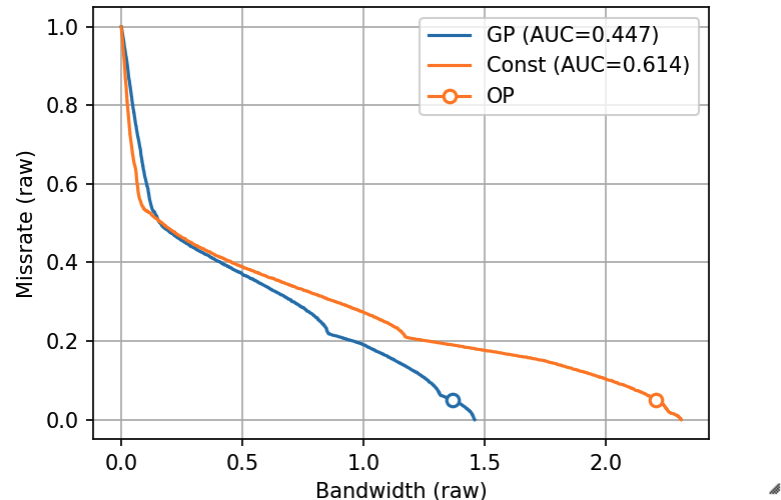}
  \captionof{figure}{UCC corresponding to the Gaussian Process (GP) task shown in Figure \ref{Fig:GainExample}}
  \label{Fig:GainExampleUCC}
\end{minipage}\hspace{.5cm}%
\begin{minipage}{.27\textwidth}
  \centering
  \includegraphics[width=4.5cm]{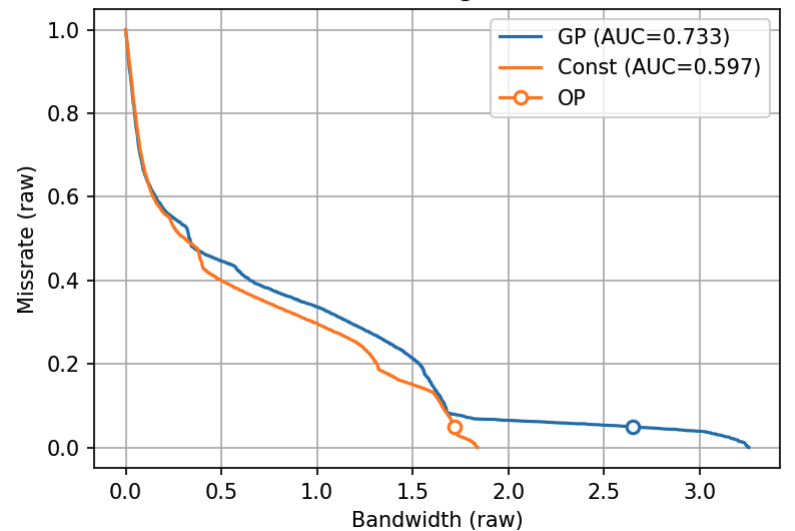}
  \captionof{figure}{UCC corresponding to the GBR QR model shown in Figure \ref{Fig:GainExample}}
  \label{Fig:GainExampleQRUCC}
\end{minipage}
\end{figure*}


\vspace{-.3cm}
\paragraph{Example from Section \ref{Sec:RelativeGain}}
Returning to the example shown in Figure \ref{Fig:GainExample} involving a Gaussian Process (GP) regressor
and a Gradient Boosting regressor (GBR), we now plot and examine 
the corresponding UCCs.
Figure \ref{Fig:GainExampleUCC} shows the GP and its constant-band reference. 
The AUUCC Gain for the GP uncertainty predictions is 57.3\%
indicating a good performance. From the Figure \ref{Fig:GainExampleUCC} we observe 
that the GP fares increasingly well as the miss rate OP decreases from 0.5 to 0.0 obtaining most of the gain. For miss rates between 0.5 and 1.0, the constant-band reference slightly outperforms the GP. 
Note that both curves exhibit a visible break around miss rate of 0.5. This point corresponds to a bandwidth at which both PIs
fully envelope the observations inside the regions with training data availability (small prediction error) while missing most of the gap regions.  

Figure \ref{Fig:GainExampleQRUCC} shows the UCC of the GBR portion of Figure \ref{Fig:GainExample}. 
As argued in Section \ref{Sec:RelativeGain}, the GBR model is inferior to the constant-band reference
which is reflected in the UCC with the AUUCCs of 0.733 and 0.597 for the GBR and the reference, respectively. 
The AUUCC Gain is -31.5\%. 

\paragraph{Introductory Example from Section \ref{Sec:Intro}}

Two BNN sampling methods, namely the Variational Inference (VI)
and the Hamiltonian Monte Carlo (HMC), resulted in an inconclusive comparison in terms of log likelihood (see Figure \ref{Fig:IntroGapExample}). 
The corresponding AUUCC gains are shown in Table \ref{Tab:AUUCCGainsForIntro}. The partial gains
are calculated focusing on a miss rate range $[0, 0.5]$ (i.e., high coverage).
\begin{table}[htb!]
\vspace{-.4cm}
\centering    
    \caption{AUUCC Gains for the VI and the HMC methods shown in Figure \ref{Fig:IntroGapExample} }
    \label{Tab:AUUCCGainsForIntro}
    \begin{tabular}{lcc}
        \toprule
        {Method} &  { \% $G_{AUUCC}$} &  { \% $G_{Partial}$} \\        
        \midrule
        VI &     -4.3 &    -45.3 \\
        HMC &     6.1 &    72.7 \\
        \bottomrule
    \end{tabular}    
\end{table}
\vspace{-.3cm}
In spite of the VI likelihood being slightly higher than that of the HMC, the quality of the HMC bounds is nevertheless higher obtaining positive gains, which aligns with the general consensus that HMC estimates are superior \cite{neal2011mcmc, Foong2020}.

The same example offers a view of another interesting phenomenon, namely a UCC cross-over, shown in Figure \ref{Fig:ZoomGap} which zooms into the gap interval, $[-1.5, 1.5]$.
\begin{figure} 
      \centering
      \includegraphics[width=14cm, height=4.5cm]{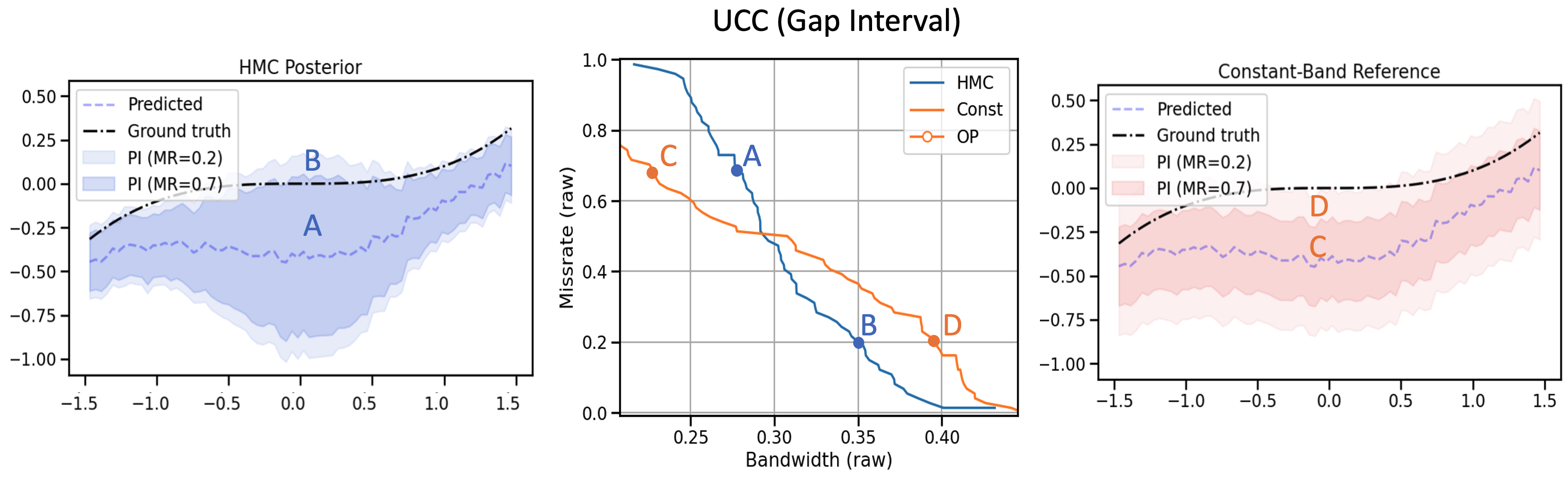}
      \caption{Zoomed section of the gap interval from Figure \ref{Fig:IntroGapExample} with UCC cross-over. MR stands for Miss Rate.}
      \label{Fig:ZoomGap}
\end{figure}
The chart highlights two distinct OPs for each curve. OPs A and C lie in a high miss rate 
range. In this regime, as shown in the data plot on the left, both PI miss 70\% of the observations, however, the 
HMC's bandwidth is higher (its center-widened PI are closer to but have not reached the observations in the center). 
Therefore the HMC UCC lies above the reference UCC. The situation is reversed for OPs B and D which operate 
at miss rates of 0.2. At this miss rate, the HMC's central-gap widening becomes beneficial and captures the observations
without excessive bandwidth, unlike the constant band, as shown on the left-hand side of 
the figure\footnote{Note that the gains reported in Table \ref{Tab:AUUCCGainsForIntro} were calculated on the full 
data sample, not on the zoomed-in excerpt used in this discussion.}.

The above cross-over phenomenon underscores the power of insight in the UCC:
Evaluating the above PIs at any fixed operating point would tell an incomplete story, the conclusion of which depends on the operating point chosen.
%

\begin{figure*} [h]
       \centering
       \includegraphics[width=14cm, height=4.cm]{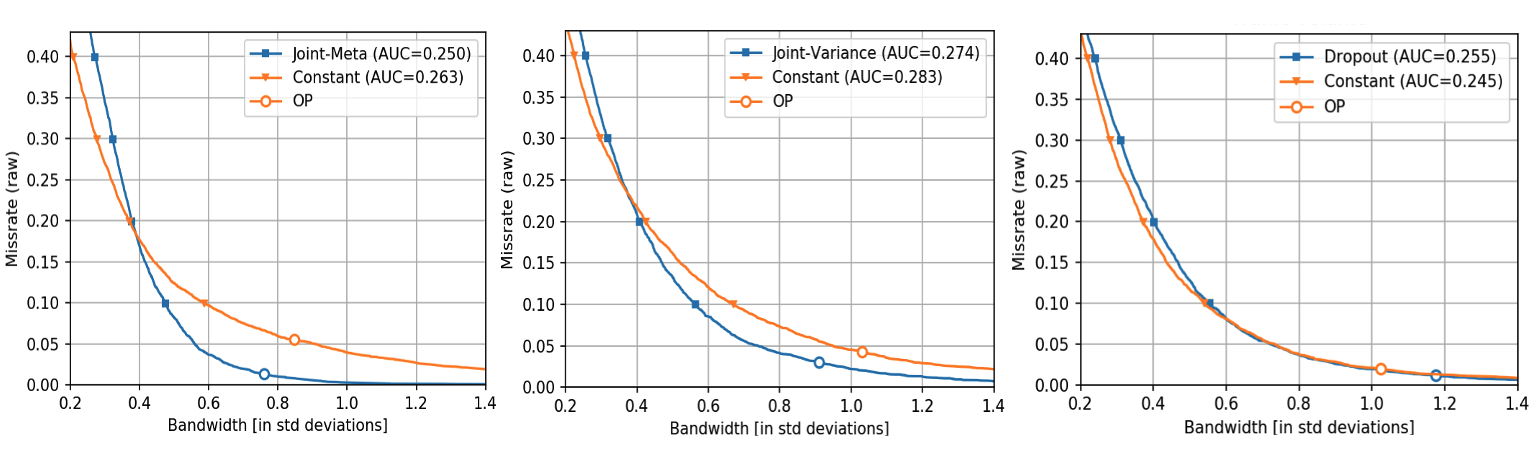}
       \caption{UCCs obtained on the Traffic dataset using an LSTM-based target predictor and three different techniques}
       \label{Fig:TrafficUCC}
\end{figure*}
\vspace{-.3cm}
\subsection{Real-World Datasets}
\label{Sec:RealWorldDatasets}
\vspace{-.3cm}
We chose three real-world datasets, namely 
(1) the Boston Housing Dataset\footnote{https://github.com/scikit-learn/scikit-learn/blob/master/sklearn/datasets/data/boston\_house\_prices.csv}, 
(2) the Wine Quality Dataset \cite{WineQualityDatasetPaper}\footnote{https://archive.ics.uci.edu/ml/machine-learning-databases/wine-quality},
and (3) Metro Interstate Traffic Volume\footnote{https://archive.ics.uci.edu/ml/machine-learning-databases/00492/}.

For brevity we only present results of the Traffic dataset in the main paper referring the reader to 
the Appendix for a comprehensive set of UCC results. 

For the sequential Traffic Volume task we employed an LSTM-based sequence-to-sequence architecture developed in \cite{ICMLAnon}. In this study, several approaches of generating prediction intervals were taken: (1) ``Joint Meta'' (JMS) model with an internal component learning to generate symmetric prediction intervals. The JMS component has a recursive structure trained jointly with the target  predictor; (2) ``Joint-Variance'' model (JMV) - a similar sequence-to-sequence model with additional output nodes implicitly modeling the heteroskedastic variance 
via a gaussian likelihood loss, similar to \cite{Lakshminarayanan2017, Kendall2017_whatuncertainty, Oh2020_crowdcounting}; (3) Variational Dropout (DOMS) - a model with an LSTM structure trained according to \cite{gal2015theoretically} to generate multiple output sequences. While the mean of the multiple outputs serves as the target predictor, their variance determines the (symmetric) prediction intervals. Note that all three cases represent time-varying, symmetric prediction intervals.
Figure \ref{Fig:TrafficUCC} shows the three UCCs. 
It can be observed that while the JMS and JMV models outperform their constant baselines in the low miss rate range (< 0.2), the DOMS case performs comparably to the constant for most of the range. 
The overall AUUCC gains are as follows: 4.9\% (JMS), 3.2\% (JMV), and -4.0\% (DOMS). 

"Cross-overs" can be seen in all cases, underscoring the need for a systematic visualization the UCC offers.  


It may be argued that certain UCC ranges may not be of practical interest, for example, miss rates of 50\% or more may be considered too high. We believe that, as an OP-agnostic tool, the UCC should include the full range in absence of  a-priori knowledge. It is conceivable that certain applications may be tuned to ``reject'' (miss) 50\% or more observations if their inclusion is associated with a high cost (e.g. accepting a potential anomaly as normal). The argument may only be fully resolved by having a concrete set of applications at hand.  Then the UCC analysis can be adapted by focusing on partial AUUCC as discussed above. 
\vspace{-.2cm}
\section {Conclusions}
\vspace{-.2cm}
In this work we introduced the Uncertainty Characteristics Curve (UCC) in conjunction with 
a gain metric relative to constant-band references, and demonstrated its power in diagnostics of prediction intervals. The UCC is formed by varying a scaling-based calibration applied to the prediction intervals,
thus characterizing their quality in an operating point agnostic manner. 
We analyzed the area under the curve and tied its
value to certain probabilistic expectations of the metrics involved. 
In several experimental case studies, the UCC 
was shown to provide important insights in terms of both the AUUCC gain metrics and the operating characteristics along the calibration range. 
With the release of the corresponding code, we believe the UCC will become a valuable new addition in the diagnostic toolbox for ML researchers and practitioners alike.  

\section{Acknowledgements}
We thank Karthikeyan Shanmugam of IBM Research and David Klus\'a\v{c}ek of the Faculty of Mathematics and Physics of the Charles University in Prague for helpful discussions regarding the proofs in this manuscript.

\bibliography{main}

\begin{thebibliography}{42}
\providecommand{\natexlab}[1]{#1}
\providecommand{\url}[1]{\texttt{#1}}
\expandafter\ifx\csname urlstyle\endcsname\relax
  \providecommand{\doi}[1]{doi: #1}\else
  \providecommand{\doi}{doi: \begingroup \urlstyle{rm}\Url}\fi

\bibitem[{Arnold} et~al.(2019){Arnold}, {Bellamy}, {Hind}, {Houde}, {Mehta},
  {Mojsilovi\'c}, {Nair}, {Ramamurthy}, {Olteanu}, {Piorkowski}, {Reimer},
  {Richards}, {Tsay}, and {Varshney}]{Arnold2019_factsheets}
M.~{Arnold}, R.~K.~E. {Bellamy}, M.~{Hind}, S.~{Houde}, S.~{Mehta},
  A.~{Mojsilovi\'c}, R.~{Nair}, K.~N. {Ramamurthy}, A.~{Olteanu},
  D.~{Piorkowski}, D.~{Reimer}, J.~{Richards}, J.~{Tsay}, and K.~R. {Varshney}.
\newblock Factsheets: Increasing trust in ai services through supplier's
  declarations of conformity.
\newblock \emph{IBM Journal of Research and Development}, 63\penalty0
  (4/5):\penalty0 6:1--6:13, July 2019.
\newblock ISSN 0018-8646.
\newblock \doi{10.1147/JRD.2019.2942288}.

\bibitem[Authors(2020)]{ICMLAnon}
A.~Authors.
\newblock Anonymized.
\newblock In \emph{Anonymized copy included in the Supplementary Material. To
  appear in arXiv;}. arxiv, 2020.

\bibitem[Begoli et~al.(2019)Begoli, Bhattacharya, and Kusnezov]{Begoli2019}
E.~Begoli, T.~Bhattacharya, and D.~Kusnezov.
\newblock The need for uncertainty quantification in machine-assisted medical
  decision making.
\newblock \emph{Nature Mach Intell}, 1:\penalty0 20--23, 2019.

\bibitem[Bi and Bennett(2003)]{Bi_REC_2003}
J.~Bi and K.~P. Bennett.
\newblock Regression error characteristic curves.
\newblock In \emph{Proceedings of the Twentieth International Conference on
  International Conference on Machine Learning}, ICML’03, page 43–50. AAAI
  Press, 2003.
\newblock ISBN 1577351894.

\bibitem[Bradley et~al.(2019)Bradley, Demargne, and Franz]{Bradley2019}
A.~A. Bradley, J.~Demargne, and K.~J. Franz.
\newblock \emph{Attributes of Forecast Quality}, pages 849--892.
\newblock Springer Berlin Heidelberg, Berlin, Heidelberg, 2019.
\newblock ISBN 978-3-642-39925-1.
\newblock \doi{10.1007/978-3-642-39925-1_2}.
\newblock URL \url{https://doi.org/10.1007/978-3-642-39925-1_2}.

\bibitem[Brier and Allen(1951)]{Brier1951}
G.~W. Brier and R.~A. Allen.
\newblock \emph{Verification of Weather Forecasts}, pages 841--848.
\newblock American Meteorological Society, Boston, MA, 1951.
\newblock ISBN 978-1-940033-70-9.
\newblock \doi{10.1007/978-1-940033-70-9_68}.
\newblock URL \url{https://doi.org/10.1007/978-1-940033-70-9_68}.

\bibitem[Casella and Berger(2001)]{CasellaBook}
G.~Casella and R.~Berger.
\newblock \emph{Statistical Inference}.
\newblock {Duxbury Resource Center}, June 2001.
\newblock ISBN 0534243126.

\bibitem[Chen et~al.(2019)Chen, Navr{\'{a}}til, Iyengar, and
  Shanmugam]{Chen2019_linearprobes}
T.~Chen, J.~Navr{\'{a}}til, V.~Iyengar, and K.~Shanmugam.
\newblock Confidence scoring using whitebox meta-models with linear classifier
  probes.
\newblock In \emph{The 22nd International Conference on Artificial Intelligence
  and Statistics, {AISTATS} 2019, 16-18 April 2019, Naha, Okinawa, Japan},
  pages 1467--1475, 2019.

\bibitem[Cortez et~al.(2009)Cortez, Cerdeira, Almeida, Matos, and
  Reis]{WineQualityDatasetPaper}
P.~Cortez, A.~Cerdeira, F.~Almeida, T.~Matos, and J.~Reis.
\newblock Modeling wine preferences by data mining from physicochemical
  properties.
\newblock \emph{Decis. Support Syst.}, 47\penalty0 (4):\penalty0 547–553,
  Nov. 2009.
\newblock ISSN 0167-9236.
\newblock \doi{10.1016/j.dss.2009.05.016}.
\newblock URL \url{https://doi.org/10.1016/j.dss.2009.05.016}.

\bibitem[{Der Kiureghian} and Ditlevsen(2009)]{Kiureghian09}
A.~{Der Kiureghian} and O.~Ditlevsen.
\newblock Aleatoric or epistemic? does it matter?
\newblock \emph{Structural Safety}, 31\penalty0 (2):\penalty0 105--112, 2009.
\newblock ISSN 0167-4730.
\newblock \doi{10.1016/j.strusafe.2008.06.020}.

\bibitem[Dunsmore(1968)]{Dunsmore1968}
I.~R. Dunsmore.
\newblock A bayesian approach to calibration.
\newblock \emph{Journal of the Royal Statistical Society: Series B
  (Methodological)}, 30\penalty0 (2):\penalty0 396--405, 1968.
\newblock \doi{10.1111/j.2517-6161.1968.tb00740.x}.
\newblock URL
  \url{https://rss.onlinelibrary.wiley.com/doi/abs/10.1111/j.2517-6161.1968.tb00740.x}.

\bibitem[Dwass(1957)]{dwass1957_permutationtest}
M.~Dwass.
\newblock Modified randomization tests for nonparametric hypotheses.
\newblock \emph{Ann. Math. Statist.}, 28\penalty0 (1):\penalty0 181--187, 03
  1957.
\newblock \doi{10.1214/aoms/1177707045}.

\bibitem[Fawcett(2006)]{Fawcett06}
T.~Fawcett.
\newblock An introduction to roc analysis.
\newblock \emph{Pattern Recognition Letters}, 27\penalty0 (8):\penalty0 861 --
  874, 2006.
\newblock ISSN 0167-8655.
\newblock \doi{https://doi.org/10.1016/j.patrec.2005.10.010}.
\newblock URL
  \url{http://www.sciencedirect.com/science/article/pii/S016786550500303X}.
\newblock ROC Analysis in Pattern Recognition.

\bibitem[Foong et~al.(2020)Foong, Burt, Li, and Turner]{Foong2020}
A.~Foong, D.~Burt, Y.~Li, and R.~Turner.
\newblock On the expressiveness of approximate inference in bayesian neural
  networks.
\newblock In H.~Larochelle, M.~Ranzato, R.~Hadsell, M.~F. Balcan, and H.~Lin,
  editors, \emph{Advances in Neural Information Processing Systems}, volume~33,
  pages 15897--15908. Curran Associates, Inc., 2020.

\bibitem[Gal and Ghahramani(2016)]{gal2015theoretically}
Y.~Gal and Z.~Ghahramani.
\newblock A theoretically grounded application of dropout in recurrent neural
  networks.
\newblock In D.~D. Lee, M.~Sugiyama, U.~V. Luxburg, I.~Guyon, and R.~Garnett,
  editors, \emph{Advances in Neural Information Processing Systems 29}, pages
  1019--1027. Curran Associates, Inc., 2016.

\bibitem[Gneiting et~al.(2007)Gneiting, Balabdaoui, and Raftery]{Gneiting2007}
T.~Gneiting, F.~Balabdaoui, and A.~E. Raftery.
\newblock Probabilistic forecasts, calibration and sharpness.
\newblock \emph{Journal of the Royal Statistical Society: Series B (Statistical
  Methodology)}, 69\penalty0 (2):\penalty0 243--268, 2007.
\newblock \doi{10.1111/j.1467-9868.2007.00587.x}.
\newblock URL
  \url{https://rss.onlinelibrary.wiley.com/doi/abs/10.1111/j.1467-9868.2007.00587.x}.

\bibitem[Guo et~al.(2017)Guo, Pleiss, Sun, and Weinberger]{Guo2017}
C.~Guo, G.~Pleiss, Y.~Sun, and K.~Q. Weinberger.
\newblock On calibration of modern neural networks.
\newblock In \emph{Proceedings of the 34th International Conference on Machine
  Learning - Volume 70}, ICML’17, page 1321–1330. JMLR.org, 2017.

\bibitem[Jiang et~al.(2018)Jiang, Kim, Guan, and Gupta]{Jiang2018_trust}
H.~Jiang, B.~Kim, M.~Guan, and M.~Gupta.
\newblock To trust or not to trust a classifier.
\newblock In S.~Bengio, H.~Wallach, H.~Larochelle, K.~Grauman, N.~Cesa-Bianchi,
  and R.~Garnett, editors, \emph{Advances in Neural Information Processing
  Systems 31}, pages 5541--5552. Curran Associates, Inc., 2018.

\bibitem[Kendall and Gal(2017)]{Kendall2017_whatuncertainty}
A.~Kendall and Y.~Gal.
\newblock What uncertainties do we need in bayesian deep learning for computer
  vision?
\newblock In I.~Guyon, U.~V. Luxburg, S.~Bengio, H.~Wallach, R.~Fergus,
  S.~Vishwanathan, and R.~Garnett, editors, \emph{Advances in Neural
  Information Processing Systems 30}, pages 5574--5584. Curran Associates,
  Inc., 2017.

\bibitem[Kingma et~al.(2015)Kingma, Salimans, and
  Welling]{kingma2015variational}
D.~P. Kingma, T.~Salimans, and M.~Welling.
\newblock Variational dropout and the local reparameterization trick.
\newblock \emph{arXiv preprint arXiv:1506.02557}, 2015.

\bibitem[Koenker and Bassett(1978)]{KoenkerQuantileRegression78}
R.~W. Koenker and G.~Bassett.
\newblock Regression quantiles.
\newblock \emph{Econometrica}, 46\penalty0 (1):\penalty0 33--50, 1978.
\newblock URL
  \url{https://EconPapers.repec.org/RePEc:ecm:emetrp:v:46:y:1978:i:1:p:33-50}.

\bibitem[Kuleshov et~al.(2018)Kuleshov, Fenner, and
  Ermon]{KuleshovAccurateUncertainties}
V.~Kuleshov, N.~Fenner, and S.~Ermon.
\newblock Accurate uncertainties for deep learning using calibrated regression.
\newblock In J.~Dy and A.~Krause, editors, \emph{Proceedings of the 35th
  International Conference on Machine Learning}, volume~80 of \emph{Proceedings
  of Machine Learning Research}, pages 2796--2804, Stockholmsmässan, Stockholm
  Sweden, 10--15 Jul 2018. PMLR.
\newblock URL \url{http://proceedings.mlr.press/v80/kuleshov18a.html}.

\bibitem[Lakshminarayanan et~al.(2017)Lakshminarayanan, Pritzel, and
  Blundell]{Lakshminarayanan2017}
B.~Lakshminarayanan, A.~Pritzel, and C.~Blundell.
\newblock Simple and scalable predictive uncertainty estimation using deep
  ensembles.
\newblock In I.~Guyon, U.~V. Luxburg, S.~Bengio, H.~Wallach, R.~Fergus,
  S.~Vishwanathan, and R.~Garnett, editors, \emph{Advances in Neural
  Information Processing Systems 30}, pages 6402--6413. Curran Associates,
  Inc., 2017.

\bibitem[Naeini et~al.(2015)Naeini, Cooper, and Hauskrecht]{Naeini15}
M.~P. Naeini, G.~F. Cooper, and M.~Hauskrecht.
\newblock Obtaining well calibrated probabilities using bayesian binning.
\newblock In \emph{Proceedings of the Twenty-Ninth AAAI Conference on
  Artificial Intelligence}, AAAI’15, page 2901–2907. AAAI Press, 2015.
\newblock ISBN 0262511290.

\bibitem[Narasimhan and Agarwal(2013)]{PartialAUC2013}
H.~Narasimhan and S.~Agarwal.
\newblock A structural {SVM} based approach for optimizing partial auc.
\newblock In S.~Dasgupta and D.~McAllester, editors, \emph{Proceedings of
  Machine Learning Research}, volume~28, pages 516--524, Atlanta, Georgia, USA,
  17--19 Jun 2013. PMLR.
\newblock URL \url{http://proceedings.mlr.press/v28/narasimhan13.html}.

\bibitem[Neal et~al.(2011)]{neal2011mcmc}
R.~M. Neal et~al.
\newblock Mcmc using hamiltonian dynamics.
\newblock \emph{Handbook of markov chain monte carlo}, 2\penalty0
  (11):\penalty0 2, 2011.

\bibitem[Niculescu-Mizil and Caruana(2005)]{ReliabilityDiagrams2005}
A.~Niculescu-Mizil and R.~Caruana.
\newblock Predicting good probabilities with supervised learning.
\newblock In \emph{Proceedings of the 22nd International Conference on Machine
  Learning}, ICML '05, page 625–632, New York, NY, USA, 2005. Association for
  Computing Machinery.
\newblock ISBN 1595931805.
\newblock \doi{10.1145/1102351.1102430}.
\newblock URL \url{https://doi.org/10.1145/1102351.1102430}.

\bibitem[{Nix} and {Weigend}(1994)]{Nix1994_variancemodel}
D.~A. {Nix} and A.~S. {Weigend}.
\newblock Estimating the mean and variance of the target probability
  distribution.
\newblock In \emph{Proceedings of 1994 IEEE International Conference on Neural
  Networks (ICNN'94)}, volume~1, pages 55--60 vol.1, June 1994.
\newblock \doi{10.1109/ICNN.1994.374138}.

\bibitem[Oh et~al.(2020)Oh, Olsen, and Ramamurthy]{Oh2020_crowdcounting}
M.~Oh, P.~A. Olsen, and K.~N. Ramamurthy.
\newblock Crowd counting with decomposed uncertainty.
\newblock In \emph{AAAI Conference on Artificial Intelligence}, 2020.

\bibitem[Papoulis and Saunders(1989)]{Papoulis1989}
A.~Papoulis and H.~Saunders.
\newblock {Probability, Random Variables and Stochastic Processes (2nd
  Edition)}.
\newblock \emph{Journal of Vibration, Acoustics, Stress, and Reliability in
  Design}, 111\penalty0 (1):\penalty0 123--125, 01 1989.
\newblock ISSN 0739-3717.
\newblock \doi{10.1115/1.3269815}.

\bibitem[Pedregosa et~al.(2011)Pedregosa, Varoquaux, Gramfort, Michel, Thirion,
  Grisel, Blondel, Prettenhofer, Weiss, Dubourg, Vanderplas, Passos,
  Cournapeau, Brucher, Perrot, and Duchesnay]{scikit-learn}
F.~Pedregosa, G.~Varoquaux, A.~Gramfort, V.~Michel, B.~Thirion, O.~Grisel,
  M.~Blondel, P.~Prettenhofer, R.~Weiss, V.~Dubourg, J.~Vanderplas, A.~Passos,
  D.~Cournapeau, M.~Brucher, M.~Perrot, and E.~Duchesnay.
\newblock Scikit-learn: Machine learning in {P}ython.
\newblock \emph{Journal of Machine Learning Research}, 12:\penalty0 2825--2830,
  2011.

\bibitem[Rinne(2011)]{Rinne2011}
H.~Rinne.
\newblock \emph{Location-Scale Distributions}, pages 752--754.
\newblock Springer Berlin Heidelberg, Berlin, Heidelberg, 2011.
\newblock ISBN 978-3-642-04898-2.
\newblock \doi{10.1007/978-3-642-04898-2_341}.
\newblock URL \url{https://doi.org/10.1007/978-3-642-04898-2_341}.

\bibitem[Shen et~al.(2018)Shen, Wang, and Chen]{Shen2018_mentionsbandwidth}
Y.~Shen, X.~Wang, and J.~Chen.
\newblock Wind power forecasting using multi-objective evolutionary algorithms
  for wavelet neural network-optimized prediction intervals.
\newblock \emph{Applied Sciences}, 8\penalty0 (2):\penalty0 185, Jan 2018.
\newblock ISSN 2076-3417.
\newblock \doi{10.3390/app8020185}.

\bibitem[Snoek et~al.(2019)Snoek, Ovadia, Fertig, Lakshminarayanan, Nowozin,
  Sculley, Dillon, Ren, and Nado]{ovadia2019trustinuncertainty}
J.~Snoek, Y.~Ovadia, E.~Fertig, B.~Lakshminarayanan, S.~Nowozin, D.~Sculley,
  J.~Dillon, J.~Ren, and Z.~Nado.
\newblock Can you trust your model's uncertainty? evaluating predictive
  uncertainty under dataset shift.
\newblock In \emph{Advances in Neural Information Processing Systems}, pages
  13969--13980, 2019.

\bibitem[Song et~al.(2019)Song, Diethe, Kull, and
  Flach]{SongDistributionCalibration2019}
H.~Song, T.~Diethe, M.~Kull, and P.~Flach.
\newblock Distribution calibration for regression.
\newblock In K.~Chaudhuri and R.~Salakhutdinov, editors, \emph{International
  Conference on Machine Learning, 9-15 June 2019, Long Beach, California, USA},
  Proceedings of Machine Learning Research, pages 5897--5906. Proceedings of
  Machine Learning Research, May 2019.

\bibitem[Steel and Torrie(1960)]{Steel1960}
R.~G.~D. Steel and J.~H. Torrie.
\newblock \emph{Principles and procedures of statistics.}
\newblock McGraw-Hill Book Company, Inc., New York, Toronto, London, 1960.

\bibitem[Tagasovska and Lopez-Paz(2019)]{Tagasovska2019}
N.~Tagasovska and D.~Lopez-Paz.
\newblock Single-model uncertainties for deep learning.
\newblock In H.~Wallach, H.~Larochelle, A.~Beygelzimer, F.~d'~Alch\'{e}-Buc,
  E.~Fox, and R.~Garnett, editors, \emph{Advances in Neural Information
  Processing Systems 32}, pages 6417--6428. Curran Associates, Inc., 2019.
\newblock URL
  \url{http://papers.nips.cc/paper/8870-single-model-uncertainties-for-deep-learning.pdf}.

\bibitem[Tran et~al.(2020)Tran, Neiswanger, Yoon, Zhang, Xing, and
  Ulissi]{tran2020methods}
K.~Tran, W.~Neiswanger, J.~Yoon, Q.~Zhang, E.~Xing, and Z.~W. Ulissi.
\newblock Methods for comparing uncertainty quantifications for material
  property predictions.
\newblock \emph{Machine Learning: Science and Technology}, 1\penalty0
  (2):\penalty0 025006, 2020.

\bibitem[URL(2021)]{SineExampleScikit}
.~URL.
\newblock Prediction intervals for gradient boosting regression.
\newblock
  \url{https://scikit-learn.org/stable/auto_examples/ensemble/plot_gradient_boosting_quantile.html},
  2021.

\bibitem[Yao et~al.(2019)Yao, Pan, Ghosh, and Doshi-Velez]{yao2019quality}
J.~Yao, W.~Pan, S.~Ghosh, and F.~Doshi-Velez.
\newblock Quality of uncertainty quantification for bayesian neural network
  inference.
\newblock \emph{ICML workshop on uncertainty in deep learning}, 2019.

\bibitem[Zadrozny and Elkan(2001)]{Zadrozny01obtainingcalibrated}
B.~Zadrozny and C.~Elkan.
\newblock Obtaining calibrated probability estimates from decision trees and
  naive bayesian classifiers.
\newblock In \emph{In Proceedings of the Eighteenth International Conference on
  Machine Learning}, pages 609--616. Morgan Kaufmann, 2001.

\bibitem[Zadrozny and Elkan(2002)]{Zadrozny2002}
B.~Zadrozny and C.~Elkan.
\newblock Transforming classifier scores into accurate multiclass probability
  estimates.
\newblock In \emph{Proceedings of the Eighth ACM SIGKDD International
  Conference on Knowledge Discovery and Data Mining}, KDD '02, pages 694--699,
  New York, NY, USA, 2002. ACM.
\newblock ISBN 1-58113-567-X.

\end{thebibliography}
\bibliographystyle{abbrvnat}

\clearpage

\paragraph{Appendix to Paper: Uncertainty Characteristics Curves: A Systematic Assessment of Prediction Intervals}

\section{Algorithm to calculate the UCC}
\begin{algorithm}[H]
  \caption{Algorithm to calculate the UCC}
  \label{Alg:UCC}
\begin{algorithmic}
    \State {\bfseries Input:}
  Ground truth, predictions, uncertainty estimates $\{y_i, \hat{y}_i, \hat{y}_i^l, \hat{y}_i^u\}_{1\leq i\leq N}$ 
    \State {\bfseries Output:} Set of UCC points $\{x_i, y_i\}_{1\leq i\leq N}$
    \For{$i\leftarrow 1$ {\bfseries to} $N$}
    \State $z_i\leftarrow y_i - \hat{y}_i$ \Comment{Observed error}
    \State 
    $k_i \leftarrow   
     \begin{cases}
      \frac{z_i}{\hat{y}_i^u-\hat{y}}& \text{for}\,\,z_i\geq 0\\
      -\frac{z_i}{\hat{y}-\hat{y}_i^l}& \text{otherwise}\\
     \end{cases}
    $ \Comment{Critical scale} 
    \State $x_i \leftarrow \hat{\beta}(k_i)$  \Comment{$\hat{\beta}, \hat{\rho}$ defined in Eq. (\ref{Eq:MetricShorthand})}
    \State $y_i \leftarrow \hat{\rho}(k_i)$
    \EndFor
\end{algorithmic}
\end{algorithm}

\section{Confounding predictions and uncertainty: An example}
Suppose there is a model predicting the regression target, $\hat{y}$, as well as a gaussian uncertainty with mean $\hat{y}$ and variance $\hat{\sigma}^2$. The loss with respect to model parameters $\theta$ is a function of both the predictions, $\hat{y}_i$ and the uncertainties, $\hat{\sigma}_i$:
$-log P(\theta) = \frac{1}{2} \sum_i^N \frac{(y_i-\hat{y}_i)^2}{\hat{\sigma}_i^2} + \log \hat{\sigma}_i^2 + C$.
We refer to the fact that the loss combines these two predictive aspects (target and uncertainty) as "comfounding" in the main paper.   

\section{Additional Metrics}\label{App:Sec:AddMetrics}
In addition to bandwidth and miss rate, defined in Section \ref{Sec:Metrics}, 
we define two additional, related metrics as follows 
    \begin{flalign} 
        \mbox{Excess: }\nonumber\\
        \qquad\xi(V) &= \E\left[\mathbf{1}_{Y\in[\hat{Y}^l, \hat{Y}^u]}\cdot
            \min\left\{Y-\hat{Y}^l, \hat{Y}^u-Y\right\}\right]&\nonumber\\
        \qquad \hat{\xi}(\mathbf{v}) &=
        \frac{1}{N}\sum_{i:y_{i}\in[\hat{y}_{i}^l, \hat{y}_{i}^u]}
        \min\left\{y_{i}-\hat{y}_{i}^l, \hat{y}_{i}^u - y_{i}\right\}\label{Eq:Excess}&
    \end{flalign}
    \begin{flalign} 
        \mbox{Deficit: }\nonumber\\
        \qquad\delta(V)&=\E_{p_V}\left[\mathbf{1}_{Y\notin[\hat{Y}^l, \hat{Y}^u]}\cdot
            \min\left\{\left|Y-\hat{Y}^l\right|, \left|Y-\hat{Y}^u\right|\right\}\right]& \nonumber\\
        \qquad \hat{\delta}(\mathbf{v}) &=
        \frac{1}{N}\sum_{i:y_{i}\notin[\hat{y}_{i}^l, \hat{y}_{i}^u]}
        \min\left\{|y_{i}-\hat{y}_{i}^l|, |y_{i}  -\hat{y}_{i}^u|\right\}\label{Eq:Deficit}&
    \end{flalign}
Figure \ref{Fig:Metrics} illustrates all four metrics. The relative proportion of observations lying outside the bounds (i.e., the miss rate) ignores the {\em extent} of the bounds' shortfall. The proposed Deficit, Eq. (\ref{Eq:Deficit}), captures this aspect. The type 2 cost is captured by the Bandwidth, Eq. (\ref{Eq:Bandwidth}). However, its range is indirectly compounded by the underlying variation in ${\hat{Y}}$ and ${Y}$. Therefore we propose the Excess measure, Eq. (\ref{Eq:Excess}), which also reflects the Type 2 cost, but just the portion above the minimum bandwidth necessary to include the observation.
\begin{figure}[htbp!]
       \centering
       \includegraphics[height=2.5cm]{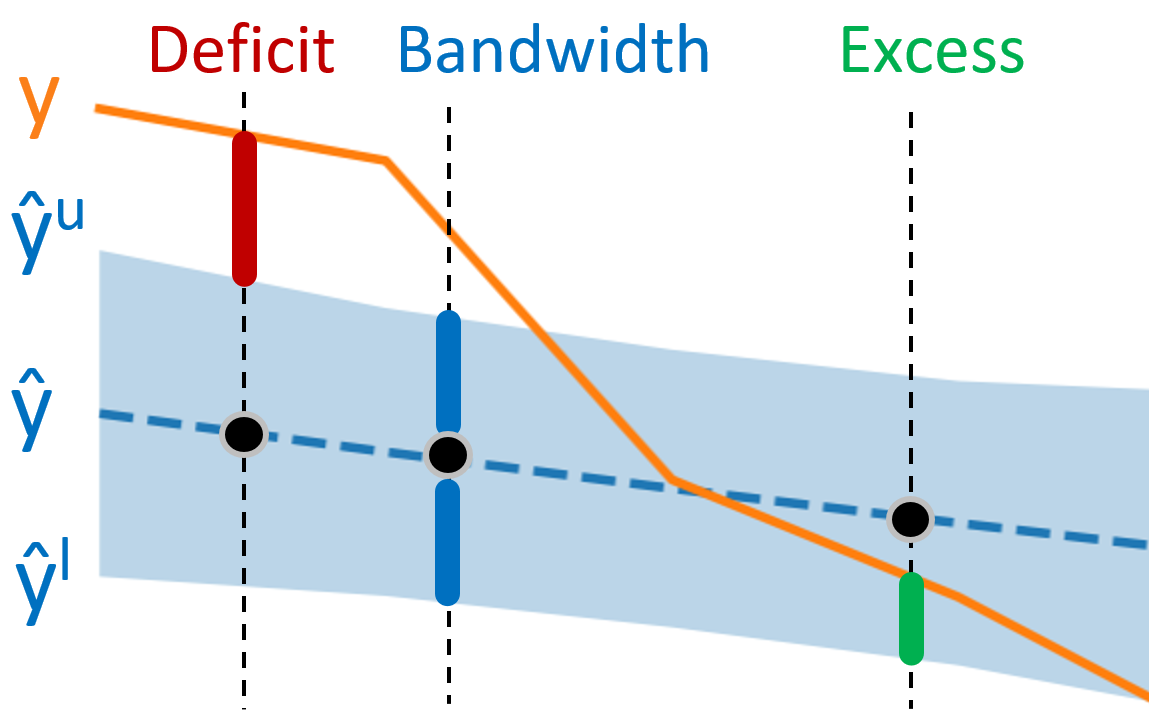}
       \caption{Bandwidth, excess, and deficit costs.}
       \label{Fig:Metrics}
\end{figure}
We will be using Excess and Deficit in reporting additional results in this document and also 
present an additional theoretical result for a UCC on Excess-Deficit coordinates. 

\section{Illustration of the effect of $\hat{y}$ on the UCC}

In Section \ref{Sec:UCC} we mention that multiple curves can be plotted in a common UCC chart
only if they are obtained using same observations, $y$, and target predictions, $\hat{y}$. 
A simple example of two PI is shown in Figure \ref{Fig:yhatmatters}. Both 
PI are constant bands, i.e., they should be characterized by an identical UCC. However, 
because they relate each to different target predictions, $\hat{y_1}$ and $\hat{y_2}$, their 
behavior with respect to the observation, $y$, is completely different. While the PI of $\hat{y_1}$ captures the observation fully at a certain critical scale leading to a "perfect" curve (UCC1), the PI of $\hat{y_2}$ incurs a positive
miss rate even at the same bandwidth. As the bandwidth varies the characteristics of the curved PI 
intersecting the observation will be non-trivial leading to a "rounder" curve (UCC2) as illustrated. 
\begin{figure}[htbp!]
       \centering
       \includegraphics[height=4.5cm]{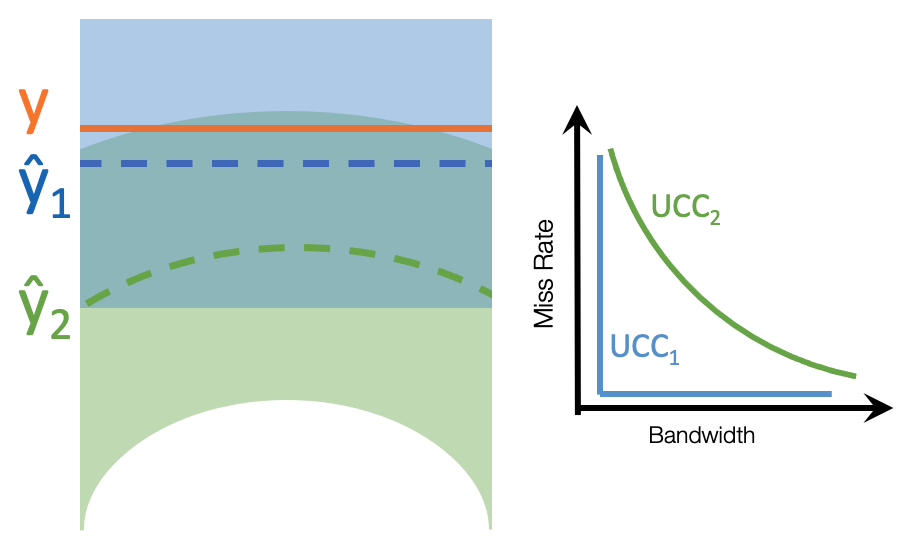}
       \caption{Illustrative example of how identical PI--constant bands in both cases--lead to different UCCs if they relate to different target predictions, $\hat{y}$}
       \label{Fig:yhatmatters}
\end{figure}
This also highlights the need for normalization which is achieved by calculating the AUUCC gain 
(see Eq. (\ref{Eq:AUUCCGain})) which, in the same example, would result in a gain of 0\% for both cases thus making the PI equivalent in terms of their quality.

\section{Proofs And Additional Results}
\label{App:Proofs}

To prove Proposition \ref{Prop:missrate} 
we use the Definition \ref{Def:CritScale} and the Lemma \ref{Lemma:1} below:

\begin{lemma}
    \label{Lemma:1}
    Choose any $v_i\in\mathbf{v}$, with $\mathbf{v}$ a sample set as defined in Section \ref{Sec:Metrics}.
    Let $k_i$ be the critical scale for $v_i$ and $K$ the scale random variable. The following holds
    \begin{equation}
        P(Y\notin[\hat{Y}-k_i\hat{Z}^l, \hat{Y}+k_i\hat{Z}^u])\equiv 1-P(K\leq k_i)\nonumber.
    \end{equation}
\end{lemma}
\begin{proof} Let $\{k_1, ..., k_N\}$ be the set of critical scales corresponding to $\{v_1,...,v_N\}$ and 
     let $k_1'\leq...\leq k_N'$ denote a sorted sequence of such scales. 
    By definition of the critical scale, for any $k_i'$ in the sequence there are exactly $i$ samples falling within,
    and $N-i$ falling outside their bounds scaled by $k_i$, i.e., 
    \begin{eqnarray}
        y_j\in[\hat{y}_j-k_i\hat{z}_j^l, \hat{y}_j+k_i\hat{z}_j^u] \qquad\forall j:k_j'\leq k_i'\nonumber\\
        y_j\notin[\hat{y}_j-k_i\hat{z}_j^l, \hat{y}_j+k_i\hat{z}_j^u] \qquad\forall j:k_j'> k_i'\nonumber
    \end{eqnarray}
    (Note that in the case of ties only the element with highest index $i$ among the tie set is considered.)
    Thus, the fraction $\frac{N-i}{N}$ corresponds to the empirical miss rate as a function of $k$ 
    (see Eq. (\ref{Eq:MetricShorthand})), which is an estimator of the miss rate probability $P(Y\notin[\hat{Y}-k_i\hat{Z}^l, \hat{Y}+k_i\hat{Z}^u])$. 
    On the other hand,
    considering $K$ a critical scale of a randomly drawn sample, $V$, the fraction $\frac{i}{N}$ is an estimator 
    for the cumulative distribution function $P_K(k_i):=P(K\leq k_i)$. Hence
    \begin{equation}
        1-P_K(k_i)\equiv P(Y\notin[\hat{Y}-k_i\hat{Z}^l, \hat{Y}+k_i\hat{Z}^u]).\label{Eq:KToMissrate}\nonumber
    \end{equation}
\end{proof}

\subsection{Proof of Proposition \ref{Prop:missrate}}
\begin{remark}
    The bandwidth $\hat{\beta}$ and excess $\hat{\xi}$ (Eq. (\ref{Eq:MetricShorthand}))
    are monotonically increasing functions of the scale $k$.
\end{remark}

\begin{proof}
    Using the fact that $B\geq 0$, its expected value can be written as follows:
    \begin{equation}
        \left<B\right>_p = \int_{0}^{\infty} b p_B(b) db = \int_{0}^{\infty}[1-P_B(b)]db
        \label{Eq:BandwidthExpectedValue}
    \end{equation}
    where 
    $P_B$ 
    denotes the cumulative distribution function of $B$. 
    The second equality uses the tail expectation formula \cite{CasellaBook}.
    
    Since $1-P_B(b) = P(B>b)$ and $\beta$ is a monotonic function of $k$
    it holds that
    \begin{equation}
    P(K>k)\equiv P(B>b).
    \end{equation}
    From the above and the Lemma \ref{Lemma:1}, it follows that $P(B>b)$ corresponds to the miss rate 
    associated with the bandwidth $b=\beta(k)$:
    \begin{equation}
        P_m(b):=P(Y\notin[\hat{Y}-k\hat{Z}^l, \hat{Y}+k\hat{Z}^u])\equiv 1-P_B(b)
    \end{equation}
     Hence, Eq. (\ref{Eq:BandwidthExpectedValue}) becomes
    \begin{equation}
        \left<B\right>_{p_B} = \int_0^\infty P_m(b)db.\label{Eq:EBintegral2}
    \end{equation}
    Given $N$ samples, $\mathbf{v}=\{v_1, ..., v_N\}$ from $p_V$, we calculate the set of critical values, $\{k_1,..., k_N\}$.
    The sorted sequence, $k_1'\leq k_2'\leq...\leq k_N'$ gives rise to a sequence of bandwidths $b_1\leq...\leq b_N$. 
    The Riemann sum corresponding to the integral (\ref{Eq:EBintegral2}) is as follows
    \begin{equation}
        S(N) = \sum_{i=1}^N P_m(b_i') \Delta b_i\label{Eq:RiemannSum}
    \end{equation}
    with a partitioning determined by the sorted observations, $b_1\leq...\leq b_N$, $\Delta b_i = b_i-b_{i-1}$, $b_0=0$, and $b_i'\in[b_{i-1}, b_i]$. Choosing $b_i'=b_i$ we rewrite the sum (\ref{Eq:RiemannSum}) as
    \begin{equation}
        S(N) = \sum_{i=1}^N \hat{\rho}(k_i') [b(k_i')-b(k_{i-1}')]\label{Eq:RiemannSum2}
    \end{equation}
    with $\hat{\rho}$ being the empirical miss rate, as per Eq. (\ref{Eq:MetricShorthand}).
    
    Eq. (\ref{Eq:RiemannSum2}) corresponds to evaluating the area under the UCC using the rectangular rule. 
    The sum will approach the expected bandwidth value in Eq. (\ref{Eq:BandwidthExpectedValue}) as $\lim_{N\rightarrow\infty} S(N)$. Thus, the empirical AUUCC is an estimator for the expected bandwidth when using bandwidth-miss rate coordinates. 
    
\end{proof}
According to Proposition \ref{Prop:missrate}, given a dataset, and given the miss rate being one of the coordinates, 
the AUUCC amounts to the other metric's average over the entire operating range. A smaller AUUCC
relates to smaller average bandwidth (or excess) measurements as the calibration scale $k$ varies, as expected from prediction intervals of higher quality. 

\begin{corollary}
\label{Corollary1}
    The area under the UCC with excess-miss rate coordinates 
    is an estimator of the expected value $\left<X\right>_{p_X}$ with
    $X$ the excess random variable and $p_X$ its density function. 
\end{corollary}

The proof of Corollary \ref{Corollary1} follows trivially from the proof of Proposition \ref{Prop:missrate} by replacing
the bandwidth variable, $B$, with excess, $X$.

\subsection{AUUCC on the Excess-Deficit Coordinates\ref{Prop:ExDef}}
\begin{proposition}\label{Prop:ExDef}
    Let $X$ and $D$ be the excess and deficit random variables generated by randomly selecting 
    a sample, $v$, determining its critical scale, $k$, and obtaining their values via Eq. (\ref{Eq:MetricShorthand}).
    Let the UCC be defined on the excess-deficit coordinates, $(\hat{\xi}, \hat{\delta})$, and the metrics 
    $\rho, \xi, \delta$ be differentiable and invertible functions. 
    The area under the UCC is an estimator of a quantity proportional to the expected value $\left<D\right>_q$ 
    with respect to a density given by $q(d)=\frac{p_D(d)}{p_X(\delta^{-1}(d))}/Q$, where $p_D$, $p_X$ denote deficit and excess
    densities and $Q$ is a normalizing constant, $Q=\int_0^\infty \frac{p_D(d)}{p_X(\delta^{-1}(d))}dd$. 
\end{proposition}
In this case, the interpretation involves an expectation of the deficit metric
proportional to a density {\em ratio} of the deficit and the excess. 

\begin{proof}
    Let the excess variable, $x=X$, be associated with the abscissa, 
    and $\delta(x)$ be the deficit function of the excess on the ordinate axis. 
    The AUUCC is 
    \begin{equation}
        \int_0^\infty \delta(x) dx.\label{Eq:EDAUC}
    \end{equation}
    Now consider the UCC a parametric curve parametrized by the miss rate, $r\in[0,1]$. 
    Let $r=\rho_X(x)$ and $r = \rho_D(d)$ where $\rho_{X,D}$ denotes a miss rate function of the excess and deficit, respectively. 
    Then $x=\rho_X^{-1}(r)$ and $d = \rho_D^{-1}(r)$, and Eq. (\ref{Eq:EDAUC}) can be rewritten as 
    \begin{equation}
        \int_1^0 \delta(\rho_X^{-1}(r))[\rho_X^{-1}]'(r)dr \label{Eq:EDAUC2}
    \end{equation}
    It is easy to show that $\frac{d}{dr}\left[\rho_X^{-1}(r)\right]=-\frac{1}{p_X(x)}$, where $p_X>0$ refers to the excess density. Hence (\ref{Eq:EDAUC2}) becomes
    \begin{equation}
        \int_0^1 \delta(\rho_X^{-1}(r)) \frac{1}{p_X(\rho_X^{-1}(r))} dr \label{Eq:EDAUC3}.
    \end{equation}
    After applying a variable change $r=\rho_D(d)$, Eq. (\ref{Eq:EDAUC3}) becomes 
    \begin{equation}
        \int_0^{\infty} d\cdot\frac{p_D(d)}{p_X(\delta^{-1}(d))}dd\label{Eq:EDAUC4}
    \end{equation}
    where $p_D$ refers to the deficit density. 
    We normalize the density ratio in Eq. (\ref{Eq:EDAUC4}) to obtain 
    \begin{equation}
        AUC = Q\cdot\int_0^{\infty}d\cdot q(d)dd\enspace \propto \left<D\right>_q \label{Eq:EDAUC5}
    \end{equation}
    whereby $q(d):=\frac{p_D(d)}{p_X(\delta^{-1}(d))}/Q$ and $Q=\int_0^{\infty}\frac{p_D(d)}{p_X(\delta^{-1}(d))}dd$.
    Thus, the Eq. (\ref{Eq:EDAUC5}) shows the AUUCC is proportional to the expected deficit with respect to the  distribution, $q$.
    Using the Riemann sum argument, similar to one in the proof to Proposition \ref{Prop:missrate}, 
    it is straight-forward to show that the empirical AUUCC is an estimator for 
    (\ref{Eq:EDAUC5}) up to the constant $Q$.
\end{proof}

In the case of Proposition \ref{Prop:ExDef}, the interpretation involves again an expectation of one of the axes' metrics, namely the deficit, however,
with respect to a distribution of a density ratio between the deficit and the excess. Similar to the previous
result, a smaller AUUCC relates to a smaller deficit average with respect to the density, $q$.
One example of such average being small would be a case where the mode of $p_D$ lies near zero deficit and 
the corresponding $p_X$ is small there, with its mode residing at higher deficits, thus concentrating the mass of $q$ around small deficit values.

Exploiting these results in the {\em optimization} of models to produce better prediction intervals appears an interesting avenue for future work.

\subsection{Special Cases of the Linear Cost Function}

\begin{remark}\label{Rem:MAE}
    Let $d_i=|\hat{y}_i - y_i|$. Given a scale $k$, and symmetric prediction bands $\hat{z}_i:=\hat{z}_i^l = \hat{z}_i^u$, the linear cost (\ref{Eq:CostFunction}) with $c=0.5$ at any operating point $k$ on the excess-deficit coordinate system corresponds to half of the mean absolute error (MAE) between the
    absolute difference and the scaled band:
        $MAE(k) = \frac{1}{N}\sum_i |d_i-k\cdot z_i|$.
\end{remark}
\begin{remark}\label{Rem:IntervalScore}
    For the choice of $f_1=\hat{\beta}, f_2=\hat{\delta}$ and $c=\frac{1}{\alpha+1}$ with $\alpha\in[0,1]$ denoting the confidence level, the symmetric cost (\ref{Eq:CostFunction}) corresponds to the well-known Interval Score (see \cite{Gneiting2007}, Section 6.2), up to a scale $\frac{\alpha+1}{\alpha}$.
    
\end{remark}

\section{Model Configurations}
The Gradient Boosting Regressors (GBR) used in Sections \ref{Sec:SyntheticData} and 
\ref{Sec:RealWorldDatasets} were trained with hyperparameter values listed 
in Table \ref{App:Tab:HyperparametersGBR}. No tuning was performed as we adopted the 
values from \cite{SineExampleScikit}. The "GBR-Weak" was set up by reducing the tree depth
to 3, and the number of trees (estimators) to 50. 

The LSTM configuration details are listed in the Appendix of \cite{ICMLAnon} - a pdf version of which is included in the Supplementary Material.

In addition, notebooks with code to reproduce all UCC results shown 
in Section \ref{Sec:SyntheticData}, \ref{Sec:RealWorldDatasets}, and this 
Appendix, are also included in the Supplementary Material. 

\label{Sec:Configurations}
\begin{table}[h]
\centering
\caption{GBR hyperparameter settings as used to produce results in 
Section \ref{Sec:SyntheticData} and \ref{Sec:RealWorldDatasets}}.
\vspace{0.15in}
\label{App:Tab:HyperparametersGBR}
\begin{tabular}{lcc}
\toprule
{\bf Hyper-}       & {\bf Where}     & {\bf Value} \\
{\bf parameter}    & {\bf used } & {\bf             } \\
\hline
Num. of estimators & GBR       & 500  \\
Max. tree depth & GBR       & 10  \\
Subsample fraction    & both      & 0.7    \\
Num. of estimators & GBR-Weak       & 50  \\
Max. tree depth & GBR-Weak       & 3  \\
Random seed    & both      & 42   \\
Learning rate    & both      & 0.1   \\
Min. samples per leaf    & both      & 9   \\

\bottomrule \\
\end{tabular}
\end{table}

\paragraph{BNN Setup for the Introductory Example}
We used a BNN with a single hidden layer, ReLU activations, and a hundred hidden units.
Synthetic data details:  Targets are generated as $y = 0.1x^3 + \epsilon$, where 
$\epsilon \sim \mathcal{N} (0, 0.25)$. Evaluated on 70 training inputs  uniformly sampled from 
$[-3.1,-1]\cup[1,3.1]$ and 154 test inputs uniformly sampled from $[-3.1,3.1]$.
Variational inference (VI): We used doubly reparameterized variational inference with the local reparameterization trick~\cite{kingma2015variational}. We used twenty MonteCarlo samples for computing the stochastic gradients of the evidence lower bound (ELBO). We used five random-restarts and selected the solution with the highest ELBO. 
Hamiltonian Monte Carlo (HMC): We used HMC with the leapfrog integrator. We sampled the momentum variable from $\mathcal{N}(0, \mathbf{I})$, and used $L = 100$ leapfrog steps. We used 50K iterations and a burnin of 40K and a thinning of interval twenty following the settings described in previous work~\cite{yao2019quality}. We used a fixed step size of $5 \times 10^{-4}$.
Test log likelihood evaluation: For both HMC and VI we used $500$ samples from the respective approximations to evaluate the test log-likelihood. Test log-likelihood is defined as:
$$
\mathbb{E}_{(x_n, y_n)\sim D}[\mathbb{E}_{q(W)}[p(y_n \mid x_n, W)]]
$$

\section{Additional Experimental Results}
\label{App:FullResults}
\subsection{Synthetic Data}

The data used to produce the results in Section \ref{Sec:SyntheticData} are generated using 
a function $x\sin x$. The training data is created by sampling this function and adding a gaussian 
noise with randomly varying variance, as in \cite{SineExampleScikit}. We set the range to be
$x\in[0,20]$. A total of 4000 training samples were used to create the GBR models, and 1000 equidistant, non-noisy samples 
sweeping the entire range of $x$ were used for evaluation. Additional 1000 held-out set samples were 
generated for tuning the operating points. The synthetic data and the GBR model predictions 
are shown in Figure \ref{Fig:App:SyntheticData}. The ground truth in the plot is the function $x\sin x$. 
A subset of the noisy training samples, the GBR target predictions, and the 
GBR prediction intervals are shown. The prediction intervals are tuned to contain, on average, 95\% 
of the ground truth. 
    \begin{figure}[htb] 
          \centering
          \includegraphics[width=8cm]{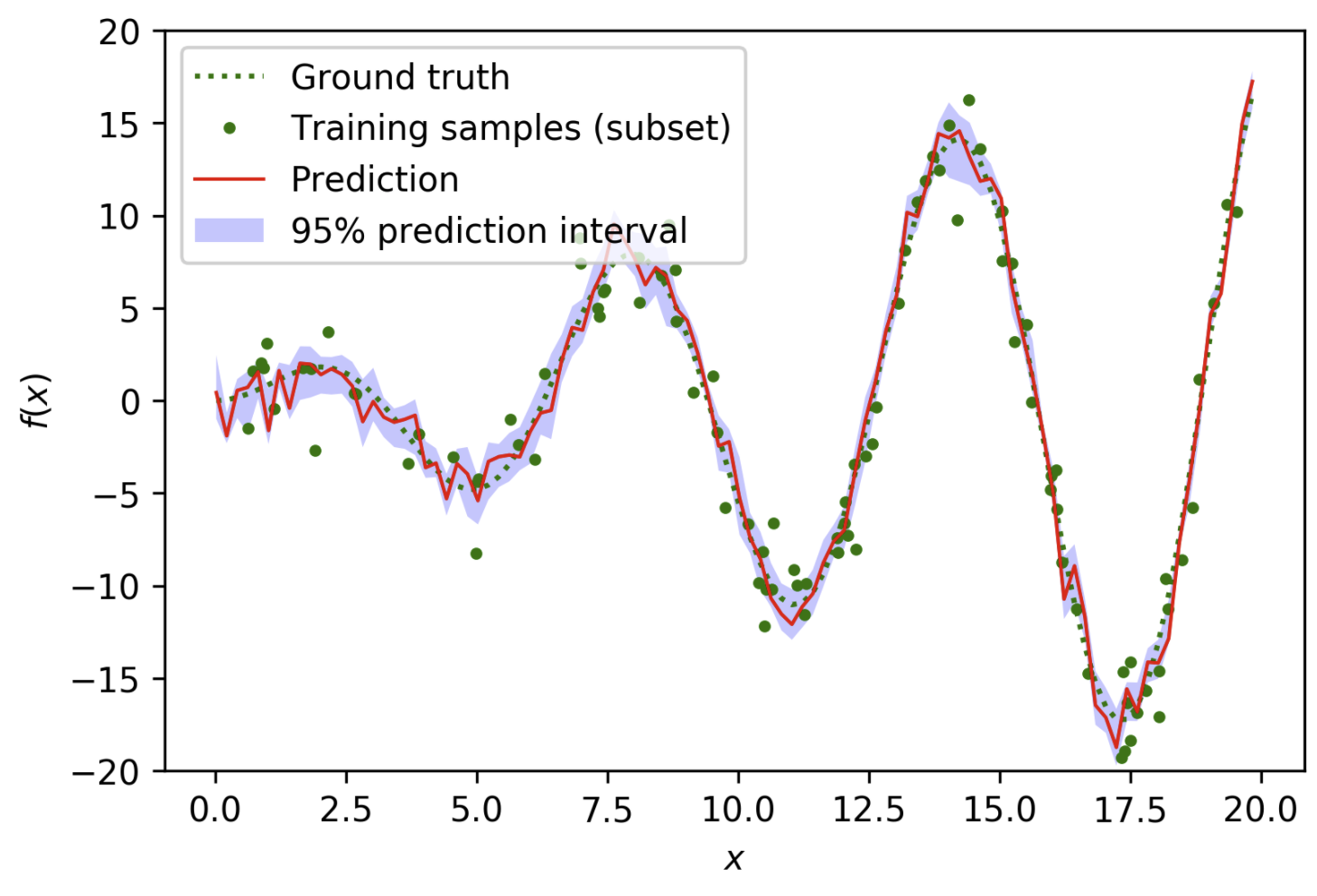}
          \caption{The synthetic dataset: a function $x\sin x$ is sampled and added noise producing training samples. The latter are used to train the GBR target and bound predictors. Also shown is the ground truth (the function $x\sin x$ without the noise), the target prediction as well as the prediction intervals.
          The ground truth is expected to lie within the bound about 95\% of the time. }
          \label{Fig:App:SyntheticData}
    \end{figure}
    \begin{table*}[ht]
        \centering    
        \caption{Summary metrics obtained on the synthetic dataset}
        \label{Tab:App:SynthDatasetTable}
        \begin{tabular}{lrrrrrrr}
        \toprule
              {\bf Model} &\multicolumn{3}{c}{{\bf Excess-Deficit}}&\multicolumn{3}{c}{{\bf Bandwidth-Miss rate}}\\
                {} &  {\bf AUUCC} &  {\bf Cost} &  {\bf Opt. Cost} &  {\bf AUUCC} &  {\bf Cost} &  {\bf Opt. Cost} &   {\bf MAE} \\
        \midrule
        GBR         &  0.138 & 0.308 &    0.099 &  0.645 & 0.318 &    0.178 & 0.750 \\
        Constant    &  0.209 & 0.287 &    0.136 &  0.819 & 0.371 &    0.245 & 0.960 \\
        GBR-Weak    &  0.311 & 0.377 &    0.182 &  0.939 & 0.405 &    0.292 & 1.284 \\
        Random      &  0.412 & 0.329 &    0.208 &  1.108 & 0.425 &    0.333 & 1.128 \\
        Eps-Perfect &  0.001 & 0.002 &    0.002 &  0.779 & 0.109 &    0.085 & 0.022 \\
        \bottomrule
        \end{tabular}
    \end{table*}
    
    The resulting UCCs are shown in the Figure \ref{Fig:SyntDatasetUCC} in Section \ref{Sec:SyntheticData} of the main paper. Here, summary metrics for the same experiment are shown in Table \ref{Tab:App:SynthDatasetTable}.
    These include the AUUCC, the Cost, the Optimum Cost, and the Mean Absolute Error (MAE). While the Cost 
    is calculated at the operating point (OP) determined on the held-out dataset, the Minimum Cost 
    is calculated on the test data themselves, thus representing a minimum achievable cost. The MAE 
    is defined in the Remark in Section \ref{Rem:MAE}. All values in Table \ref{Tab:App:SynthDatasetTable}
    are consistent with the overall model ranking apparent in the Figure \ref{Fig:SyntDatasetUCC}.
    The rather large gap between the actual and minimum cost can be attributed to the fact that the 
    held-out data set contains noise, while the test set does not. 
    This gap disappears when using a held-out set without noise.

\subsection{Real-World Datasets}

Figure \ref{Fig:HousingUCC} shows the UCCs on excess-deficit (left) and bandwidth-miss rate (right) 
coordinates. Note that here we apply axes normalization
to standard deviation units. 
Comparing the predictors across the two coordinate systems reveals interesting differences. While there is a clear separation of all the curves on the bandwidth-miss rate plot, the gap narrows for all but the weak GBR. For example, at one standard deviation of bandwidth there is a 5\% difference in miss rate between the GBR and the Weak GBR predictor - a difference comparable to one between GBR and the Meta predictor. However, when the {\em extent} of missing a target is taken into account (in the Excess-Deficit plot), the Weak GBR fares significantly worse than the GBR by same comparison, indicating the Weak GBR tends to miss targets by a greater extent.  
\begin{figure*}[htb] 
      \centering
      \includegraphics[height=5cm]{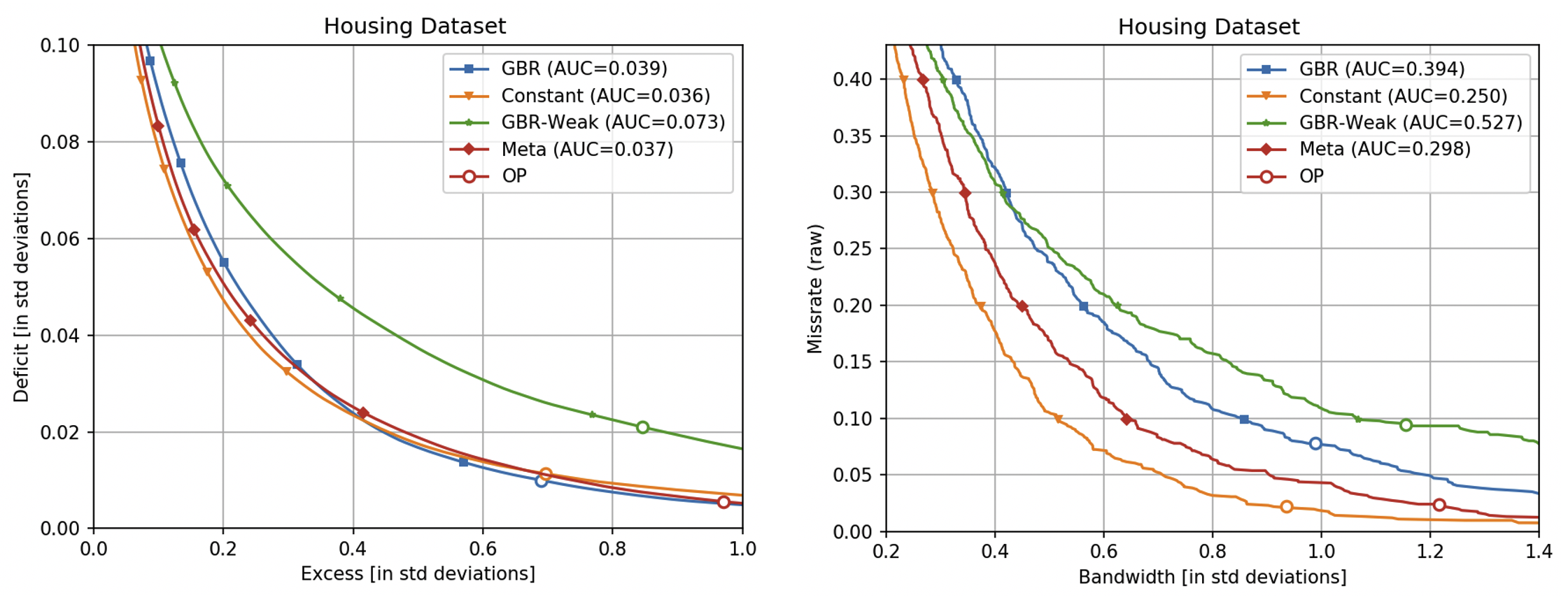}
      \caption{The UCCs on excess-deficit (left) and bandwidth-miss rate (right) coordinates for the Boston Housing data set.}
      \label{Fig:HousingUCC}
\end{figure*}
Furthermore, a curve cross-over can be seen for the GBR and the Meta model on the excess-deficit plot, comparable
to that observed on the Wine dataset (see below). This pattern is not present on the bandwidth-miss rate coordinate system - an indication that the choice of metrics is
not without consequences and should be made carefully with the eventual application in mind. 
   

\subsection{Wine Quality Data}
    We used the White Wine portion of the collection (comprising White and Red), consisting of 
    4897 samples. Furthermore, to obtain more robust results, we applied a 5-fold cross-validation 
    scheme to generate our measurements. 

    The UCC for both the excess-deficit and the bandwidth-miss rate coordinates are shown in 
    Figure \ref{Fig:App:WinesUCC}. Interestingly, there are a few significant differences between 
    the two plots: while the Meta model performs best in both coordinate systems, the GBR model falls behind the 
    constant when bandwidth and miss rate are measured. Recall that the miss rate is insensitive to 
    the extent of prediction interval excess, while the excess metric captures this. Furthermore, 
    the GBR-Weak model shows a dramatic drop at a bandwidth of about 1.3 standard deviations. 
    A further investigation revealed that the weak GBR model's predictions are almost constant
    (with a few exceptions). This combined with the fact that the wine ratings are whole numbers 
    between 1 and 10, but mostly concentrating between 5 and 7 (mode at 6), 
    there is a distinct  OP that just captures a large portion of the ground truth ratings at once. 
    This becomes visible as a jump in the UCC when miss rate is one of the coordinates. 
    
    The summary metrics are shown in Table \ref{Tab:App:WinesTable}. Significantly smaller gaps 
    can be seen between the Cost and the Opt. Cost, as compared to the synthetic data (as discussed above). 
    While most summary metrics reflect the model ranking consistently, one outlier is the 
    optimum cost for the GBR-Weak model on the bandwidth-miss rate coordinate system: based on this metric
    the GBR-Weak model might be considered the superior choice. A look at the corresponding UCCs 
    (Figure \ref{Fig:App:WinesUCC}), however, is revealing: the GBR-Weak model exhibit a dramatic 
    drop in miss rate around the bandwidth of 1.3, to the degree that it briefly dips below all other 
    curves, thus bringing the minimum achievable cost to 0.187. However, from the UCC comparison
    it is clear that such a choice may not be the most robust one. 
    
    \begin{figure*}[htb] 
          \centering
          \includegraphics[height=6cm]{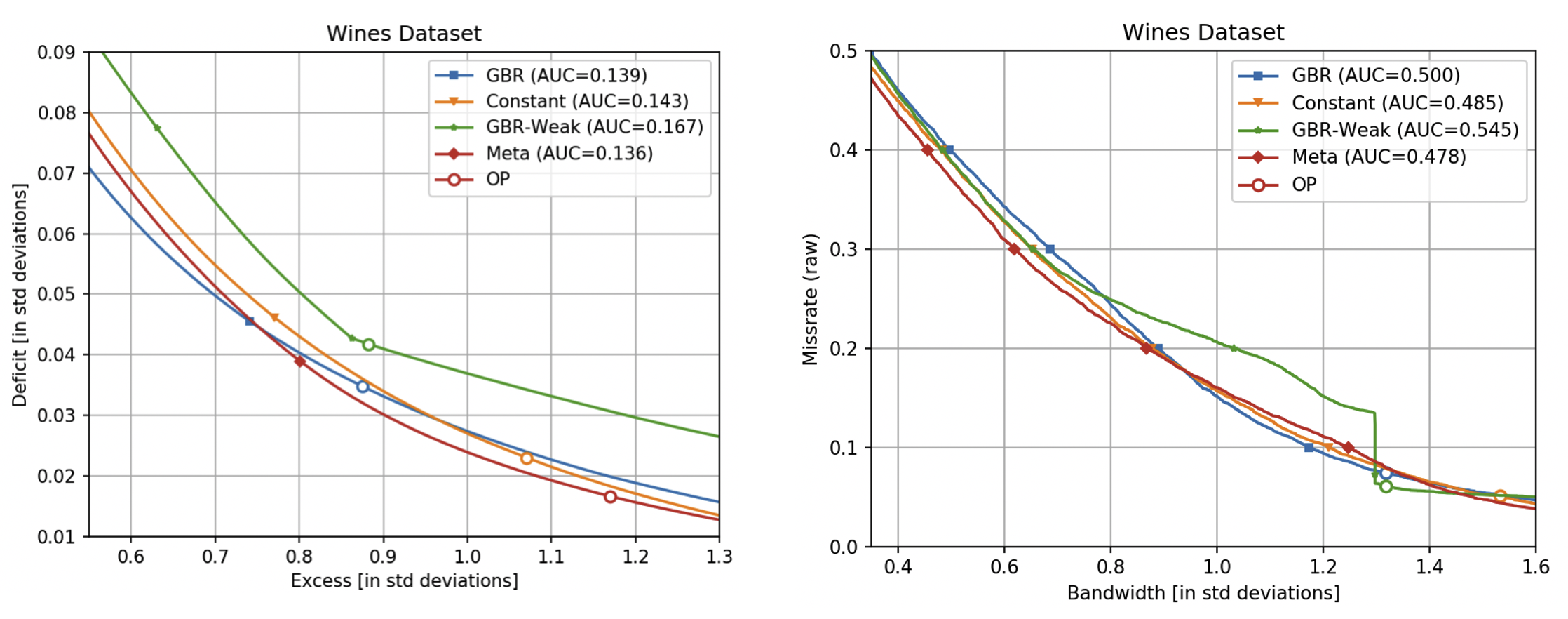}
          \caption{The UCCs on excess-deficit (left) and bandwidth-miss rate (right) coordinates for the Wine Quality dataset.}
          \label{Fig:App:WinesUCC}
    \end{figure*}
    \begin{table*}[ht]
            \centering    
            \caption{Summary metrics obtained on the Wine Quality dataset}
            \label{Tab:App:WinesTable}
            \begin{tabular}{lrrrrrrr}
            \toprule
              {\bf Model} &\multicolumn{3}{c}{{\bf Excess-Deficit}}&\multicolumn{3}{c}{{\bf Bandwidth-Miss rate}}\\
                {} &  {\bf AUUCC} &  {\bf Cost} &  {\bf Opt. Cost} &  {\bf AUUCC} &  {\bf Cost} &  {\bf Opt. Cost} &   {\bf MAE} \\
                \midrule
                GBR      &  0.139 & 0.119 &    0.115 &  0.500 & 0.199 &    0.198 & 1.031 \\
                Constant &  0.143 & 0.128 &    0.119 &  0.485 & 0.199 &    0.197 & 1.082 \\
                GBR-Weak &  0.167 & 0.126 &    0.125 &  0.545 & 0.187 &    0.187 & 1.074 \\
                Meta     &  0.136 & 0.132 &    0.115 &  0.478 & 0.196 &    0.193 & 1.294 \\
                \bottomrule
            \end{tabular}
    \end{table*}

\subsection{Boston Housing Data}
    Due to the relatively small size of the Housing dataset, the following complexity-related parameters of the GBR and GBR-Weak
    models were adjusted (compared to values listed in the Table \ref{App:Tab:HyperparametersGBR}): (1) number of estimators (trees) set to 50 and 10 for GBR and GBR-Weak, respectively, (2) maximum tree depth set to 3.
    A 10-fold cross-validation (with 2 randomized repetitions) was applied to obtain more reliable 
    measurements. 
    
    The summary metrics are shown in Table \ref{Tab:App:HousingTable}. 
    
    \begin{table*}[ht]
            \centering    
            \caption{Summary metrics obtained on the Boston Housing data set}
            \label{Tab:App:HousingTable}
            \begin{tabular}{lrrrrrrr}
            \toprule
              {\bf Model} &\multicolumn{3}{c}{{\bf Excess-Deficit}}&\multicolumn{3}{c}{{\bf Bandwidth-Miss rate}}\\
                {} &  {\bf AUUCC} &  {\bf Cost} &  {\bf Opt. Cost} &  {\bf AUUCC} &  {\bf Cost} &  {\bf Opt. Cost} &   {\bf MAE} \\
                \midrule
                GBR      &  0.039 & 0.078 &    0.061 &  0.394 & 0.169 &    0.162 & 104.762 \\
                Constant &  0.036 & 0.080 &    0.059 &  0.250 & 0.113 &    0.108 &  45.962 \\
                GBR-Weak &  0.073 & 0.103 &    0.080 &  0.527 & 0.200 &    0.196 & 112.331 \\
                Meta     &  0.037 & 0.102 &    0.061 &  0.298 & 0.143 &    0.133 & 113.851 \\
                \bottomrule
            \end{tabular}
    \end{table*}

\subsection{Traffic Volume Data}

    The full set of UCCs using both coordinate systems is shown in Figure \ref{Fig:App:TrafficUCC}.
    \begin{figure*}[htb] 
          \centering
          \includegraphics[width=13cm]{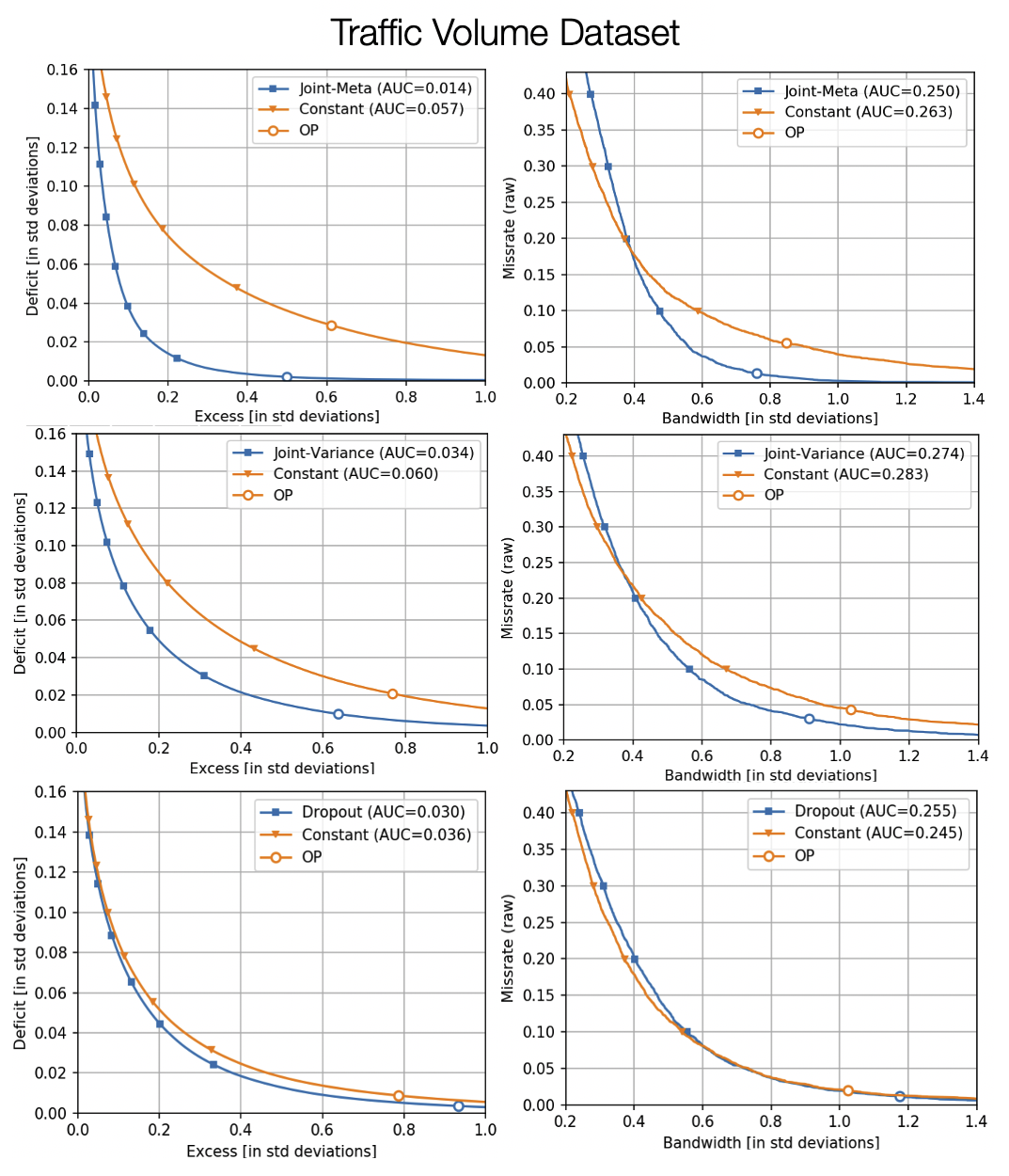}
          \caption{The UCCs on excess-deficit (left) and bandwidth-miss rate (right) coordinates for the three models (Joint-Meta, Joint-Variance, Dropout) and their respective constant baselines.}
          \label{Fig:App:TrafficUCC}
    \end{figure*}
    
    The findings from Figure \ref{Fig:App:TrafficUCC} are also reflected in the gains (see Table \ref{Tab:TrafficDatasetTable}) with the JMS 
    achieving an overall improvement of 75.5\% and 59.3\% in AUUCC ($G_{AUC}$) and optimum cost ($G_{C^*}$) over its constant
    reference. In contrast, the DOMS model obtains rel. smaller positive gains for AUUCC and the minimum cost, but negative gains for the original-calibration cost ($G_C$) and the MAE, indicating the original calibration present in the input data was suboptimal. Recall that while $G_{C}$ and $G_{MAE}$
    are determined using the original calibration, which is tuned to minimize the cost using a held-out set of training data, the gain $G_{C^*}$ is determined on the best calibration achievable using the visualized data.
    The AUUCC values in each graph are significantly different at $p<0.01$ based on the pairwise permutation test. 
    \begin{table*}[htb!]
    
        \centering    
        \caption{The sequential Traffic Volume task. Summary metrics shown as percentage gains over the corresponding constant-band baselines.}
        \label{Tab:TrafficDatasetTable}
        \begin{tabular}{lrrrr}
            \toprule
           {\bf Model} &\multicolumn{3}{c}{{\bf Excess-Deficit}}&\\
            {} &  { \% $G_{AUC}$} &  { \% $G_{C}$} &  { \% $G_{C^*}$} &  { \% $G_{MAE}$}\\        
            \midrule
            JMS &     75.5 &    40.5 &       59.3 &      21.8 \\
            JMV &     42.7 &    24.1 &       30.4 &   18.2 \\
            DOMS &     15.5 &   -11.5 &       10.0 & -17.8 \\
            \bottomrule
        \end{tabular}    
    \end{table*}

    In addition to the discussion in Section \ref{Sec:RealWorldDatasets}, we observe that an interesting 
    cross-over pattern occurs in the bandwidth-miss rate UCC plots showing the constant band outperforming 
    the relatively complex sequential models (Joint-Meta and Joint-Variance) in a high miss rate operating 
    range. As the operating miss rate gets lower (below 0.2), the model-based prediction intervals 
    begin gaining over the constant baseline significantly. This pattern does not show up on the 
    excess-deficit plots and seems to be induced, again, by the miss rate metrics ignoring the degree of excess, 
    as discussed earlier.


\end{document}